\documentclass[lettersize,journal]{IEEEtran}
\usepackage{amsmath,amsfonts}
\usepackage{algorithmic}
\usepackage{algorithm}
\usepackage{array}
\usepackage[caption=false,font=normalsize,labelfont=sf,textfont=sf]{subfig}
\usepackage{textcomp}
\usepackage{xcolor}
\usepackage{stfloats}
\usepackage{url}
\usepackage{verbatim}
\usepackage{graphicx}
\usepackage{cite}
\newtheorem{lemma}{Lemma}
\newtheorem{proposition}{Proposition}

\newtheorem{definition}{Definition}
\newtheorem{theorem}{Theorem}

\begin{document}

\title{{Truthful Online Preference Aggregation for LLM Fine-Tuning in Mobile Crowdsourcing }}

\author{Shugang~Hao,~\IEEEmembership{Member,~IEEE,} 
        and~Lingjie~Duan,~\IEEEmembership{Senior Member,~IEEE}
\thanks{Part of this work has appeared in IEEE ICASSP 2025 \cite{hao2025online}.}

\thanks{The research of Lingjie Duan was supported by Guangdong Provincial Key Lab of Integrated Communication, Sensing and Computation for Ubiquitous Internet of Things (No. 2023B1212010007).}

\thanks{Shugang Hao is with the Singapore Wireless Innovation Centre, Singapore University of Technology and Design, Singapore, 487372 Singapore. Lingjie Duan is with the Internet of Things Thrust and the Artificial Intelligence Thrust, Hong Kong University of Science and Technology, Guangzhou 511455, China.  E-mail: shugang\_hao@sutd.edu.sg, lingjieduan@hkust-gz.edu.cn. (Corresponding author: Lingjie Duan.)}

}

\maketitle

\begin{abstract}
To better serve users’ demands in mobile applications (e.g., navigation), mobile crowdsourcing platforms can iteratively align large language model (LLM)-generated content (e.g., AI-generated traffic condition predictions) with human feedback collected from crowdsourcing workers (e.g., mobile users). However, workers may strategically misreport their online preference feedback to maximize their influence or payment. Existing pipelines in mobile crowdsourcing (e.g., EM-based weight estimation) fail to identify the most accurate worker in this online setting, resulting in a linear regret $\mathcal{O}(T)$ over $T$ time slots. {In this paper, we study truthful online preference aggregation for LLM fine-tuning in mobile crowdsourcing.} We formulate a new dynamic Bayesian game to model the multi-agent online learning process between the platform and strategic mobile workers. We propose a novel online weighted aggregation mechanism that dynamically adjusts each worker’s weight in the preference aggregation according to their feedback accuracy. We prove that our mechanism ensures truthful feedback from strategic workers and achieves a sublinear regret $\mathcal{O}(\sqrt{T})$ over $T$ time slots. We further extend our mechanism to a challenging scenario with limited worker feedback per time slot, still guaranteeing a sublinear regret $\mathcal{O}(\sqrt{T})$. Experiments on LLM fine-tuning with real-world datasets further demonstrate significant performance gains of our mechanisms over benchmark schemes.
\end{abstract}


\section{Introduction}\label{S1}

To better serve users’ demands in mobile applications, mobile crowdsourcing platforms can iteratively align large language model (LLM)-generated content with human feedback collected from crowdsourcing workers (e.g., \cite{touvron2023open}, \cite{xu2024cached}). For example, navigation platforms (e.g., Waze) continuously collect human feedback on traffic conditions, routes, and system recommendations, providing human feedback required for iterative LLM alignment in dynamic mobile environments (e.g., \cite{waze_overview, waze_crowdsourcing}). Mobile conversational AI applications (e.g., ChatGPT and Gemini mobile) continuously collect users’ interaction feedback (e.g., binary ratings, response regenerations, and follow-up corrections) to assess and improve system-generated responses (e.g., \cite{ouyang2022training}, \cite{google_gemini_feedback}).

{However, recent studies find that selfish workers may strategically misreport their online preference feedback to maximize their influence or payment  (e.g., \cite{sun2024mechanism, soumalias2024truthful, park2024rlhf, conitzer2024social}). For example, there is a renowned ``wet bias" where a weather forecaster as a worker or a predictor may deliberately report an exaggerated probability of precipitation to increase the influence of his forecast in the weather forecasting platform's final prediction (e.g., \cite{roughgarden2017online}). Besides, a substantial number of Amazon Mechanical Turk (MTurk) workers are found to strategically misreport their responses to platform-elicited screening questions (e.g., falsely claiming required demographics, prior experience, or device ownership) that are explicitly requested by the platform, aiming to increase their access or weight to higher-paying tasks (e.g., \cite{chandler2017lie, kennedy2020shape}).} 
Nevertheless, existing adaptive aggregation pipelines (e.g., EM-based weighting \cite{perez2021online, perez2018adaptive} and Hedge-style online learning \cite{fanuniversal, zhou2023efficient}) largely assume passive or truthful reporting and ignore the possibility of strategic misreporting by workers. Our first research question arises:
\begin{itemize}
\item \textit{Q1. How vulnerable is the current practice of LLM fine-tuning against selfish workers?}
\end{itemize}

Later, we prove that such current practice fails to identify the most accurate worker in the online learning process. Recent work (e.g., \cite{sun2024mechanism, park2024rlhf, soumalias2024truthful, dubey2024auctions,xu2023epvisa}) proposes monetary mechanism design to elicit truthful preferences from strategic workers in LLM fine-tuning. However, such payment-based mechanisms largely focus on one-shot or offline preference elicitation and do not consider online interactions, where workers have more room to strategically misreport and play with the platform for long-term influence.

We also find some recent studies on online or iterative LLM alignment (e.g., \cite{dongrlhf}, \cite{ye2024online}, \cite{xiong2024iterative}), where the system performs LLM fine-tuning using periodically-collected human annotations. Yet, these studies focus on preference feedback from a single worker and do not address diverse human feedback from multiple workers. Further, they assume that a worker is always truthful to provide his real preference feedback, which does not capture the strategic misreport from multiple workers.

In the related literature of algorithmic game theory, there are relevant non-monetary mechanism studies on facility location games (e.g., \cite{asadi2022collaborative}, \cite{chen2024mechanism}, \cite{li2024strategyproof}), where the system aims to incentivize customers' truthful reporting of their locations to optimize facility placement. Each customer can strategically misreport his location to mislead the facility placement as close to his location (preference) as possible. The popular ``median" scheme (e.g., \cite{conitzer2024social}, \cite{wang2024positive}) to aggregate multi-agent reports is widely used to return customers' truthful reporting. Yet, later we prove that it can incur a non-vanishing regret over time. Thus, our second research question arises: 
\begin{itemize}
    \item \textit{Q2. How to design a truthful and regret-efficient mechanism against selfish workers in LLM fine-tuning for mobile crowdsourcing? }
\end{itemize}

Note that motivating truthful feedback from workers while achieving vanishing regret is highly challenging. First, workers' true preferences are hidden and may vary across time, making it difficult for the platform to detect or correct misreports and reliably infer these preferences (e.g., \cite{sun2024mechanism}). Furthermore, because the most accurate worker is unknown and must be learned online, the platform finds it hard to dynamically assign weights in a way that guarantees vanishing regret in the presence of strategic behavior.

We summarize our key novelty and main results as follows.
\begin{itemize}

\item {\textit{Truthful online preference aggregation for LLM fine-tuning in mobile crowdsourcing:} In this work, we study the design of a truthful online preference aggregation mechanism in mobile crowdsourcing applications, where heterogeneous workers may strategically misreport their preference feedback to maximize their long-term influence or payment. The aggregated preferences serve as the human-feedback dataset for iteratively fine-tuning a downstream LLM at the platform. Unlike the LLM literature studying either a single worker or offline preference feedback (e.g., \cite{sun2024mechanism}, \cite{park2024rlhf}, \cite{soumalias2024truthful}, \cite{dubey2024auctions}, \cite{dongrlhf}, \cite{ye2024online}, \cite{xiong2024iterative}), we focus on \textit{how a platform can incentivize selfish workers' truthful feedback through online aggregation mechanism design.}}

    \item \textit{Non-vanishing regrets of current  practice:} We prove that the current crowdsourcing practice (e.g., EM-based weight estimation) fails to identify the most accurate worker and can lead to a non-vanishing regret $\mathcal{O}(T)$ over $T$ time slots. Further, we prove that the popular median scheme in the algorithmic game theory literature still incurs a linear regret $\mathcal{O}(T)$.

    \item \textit{Our novel truthful online weighted aggregation mechanism:} We first formulate a new dynamic Bayesian game to model the multi-agent online learning process between the platform and strategic workers. We then propose a novel online weighted aggregation mechanism to dynamically adjust workers’ weights in the preference aggregation according to their feedback accuracy during the online learning process. We prove that our mechanism guarantees workers' truthful preference feedback and achieves a vanishing regret $\mathcal{O}(\sqrt{T})$ over $T$ time slots. {We further prove that our mechanism is responsive to new high-quality workers under the uniform step-size $\alpha$, and remains robust under bounded noisy verification of the ground-truth system state.}
    
    \item \textit{Extension to limited worker feedback:} 
    In practice, collecting feedback from multiple workers can be difficult due to cost and coordination challenges, which can in turn slow down the online learning process in LLM fine-tuning (e.g., \cite{lambert2025reinforcement}). We further extend to address a challenging scenario where only one worker's preference feedback is available per time slot. We propose a novel online mixed selection mechanism to ensure truthful feedback from any strategic worker while maintaining a sublinear regret $\mathcal{O}(\sqrt{T})$. Experiments on LLM fine-tuning based on real-world datasets further demonstrate significant performance gains of our proposed mechanisms compared to benchmark schemes. 
\end{itemize}

The rest of this paper is organized as follows. Section~\ref{S_RW} reviews related work. Section~\ref{S2} introduces the system model and the dynamic Bayesian game formulation for online mobile crowdsourcing based on LLM fine-tuning iterations. Section~\ref{S3} analyzes three common schemes used in the literature as benchmarks for our mechanism
to compare later. Section~\ref{S4} details our proposed mechanism design and analysis. Section~\ref{S5} extends the framework to limited worker feedback. Section~\ref{S6} presents experimental results on real-world datasets. Section~\ref{S7} concludes.

\section{{Related Work}}\label{S_RW}

In this section, we discuss four lines of existing work most relevant to our study.

\textit{Online LLM alignment with human feedback.} Recent studies on online or iterative LLM alignment perform LLM fine-tuning using periodically-collected human annotations to keep the policy aligned with evolving user preferences. Xiong et al. \cite{xiong2024iterative} formulate iterative preference learning as a KL-regularized optimization against a reference model and provide theoretical guarantees under the assumption of truthful single-worker feedback. Ye et al. \cite{ye2024online} extend this framework with general preference models and derive convergence rates under online interactions. Dong et al. \cite{dongrlhf} propose an online RLHF workflow that integrates reward modeling with iterative policy updates from streaming human feedback. More recent work studies online direct alignment without an explicit reward model, including iterative DPO variants and online preference optimization under distribution shift (e.g., \cite{ye2024online, xiong2024iterative}). However, these studies focus on preference feedback from a single worker and assume that the worker is always truthful, which does not capture the strategic misreport that arises in mobile crowdsourcing where heterogeneous selfish workers compete for long-term influence or payment.

\textit{Monetary mechanism design for truthful preference elicitation.} Another line of recent work proposes monetary mechanism design to elicit truthful preferences from strategic workers in LLM fine-tuning. Sun et al. \cite{sun2024mechanism} design payment mechanisms for fine-tuning with multiple reward models, ensuring incentive compatibility through monetary transfers. Soumalias et al. \cite{soumalias2024truthful} propose truthful aggregation mechanisms for LLMs in online advertising, where workers' valuations are elicited through auction-style payments. Park et al. \cite{park2024rlhf} study heterogeneous feedback aggregation under personalization with monetary incentives. Dubey et al. \cite{dubey2024auctions} further develop auction mechanisms with LLM-generated summaries, and Xu et al. \cite{xu2023epvisa} design auction mechanisms for real-time physical-virtual synchronization in human-centric metaverse settings. However, such payment-based mechanisms largely focus on one-shot or offline preference elicitation and do not consider online interactions over time, where workers have more room to strategically misreport across iterations and shape long-term outcomes. In contrast, our work focuses on non-monetary mechanism design under repeated interactions, where the platform incentivizes truthful feedback through dynamic weight adjustment rather than monetary transfers, which is more practical in mobile crowdsourcing where per-query monetary settlement to mobile users is costly to implement.

\textit{Non-monetary mechanism design in algorithmic game theory.} In the algorithmic game theory literature, the median scheme on a one-dimensional space is known to be group-strategyproof when agents have single-peaked preferences, dating back to the classical median voter result of Moulin \cite{moulin1980strategy}. This foundational result has motivated a substantial body of work on truthful mechanisms without money, particularly in facility location games (e.g., \cite{asadi2022collaborative, chen2024mechanism, li2024strategyproof}), where each customer can strategically misreport his location to bias the placement toward his own preference, and the median scheme is widely used to elicit truthful reports (e.g., \cite{conitzer2024social, wang2024positive}). Recent extensions consider obnoxious facility location \cite{li2024strategyproof} and group-fair variants under intra-group externalities \cite{wang2024positive}. A parallel line of work on peer prediction (e.g., \cite{prelec2004bayesian, miller2005eliciting}) and recent extensions to online and information-elicitation-without-verification settings (e.g., \cite{liu2017machine, witkowski2022incentive}) further studies non-monetary truthful elicitation when ground truth is unknown or delayed. However, these mechanisms are designed for static or one-shot decisions and do not address the online learning aspect with vanishing-regret guarantees. Moreover, we prove later in Section~\ref{S3.2} that the median scheme incurs a non-vanishing regret in our online mobile crowdsourcing setting, because it fails to give full credit to the most accurate worker even when all workers report truthfully.

\textit{Adaptive weighting schemes in crowdsourcing and online learning.} In the crowdsourcing and online learning literature, several adaptive weighting schemes have been proposed to dynamically aggregate feedback from multiple workers. The Dawid-Skene model \cite{dawid1979maximum} is the foundational truth-inference framework, and its EM-based extensions (e.g., \cite{perez2021online, perez2018adaptive}) treat the true outcome as a hidden variable and iteratively estimate worker reliability via Expectation-Maximization. Hedge-style online learning (e.g., \cite{fanuniversal, zhou2023efficient}) updates worker weights via exponential decay on observed losses and provides $\mathcal{O}(\sqrt{T})$ regret under truthful reporting. EXP3 and its variants (e.g., \cite{kash2024slowly, khodak2023meta}) further extend these ideas to the partial-feedback bandit setting using inverse propensity scoring. More recent work studies strategic crowdsourcing where workers may misreport to game the aggregation, including truthful peer grading \cite{shah2016double}, no-regret incentive-compatible online learning \cite{freeman2020no}, and incentive-aware federated bandits \cite{wei_truthfedban_2024}. However, existing strategic-aware schemes either rely on monetary payments, focus on offline or one-shot elicitation, or do not provide vanishing-regret guarantees under online multi-worker interactions with verifiable ground truth. We later prove in Sections~\ref{S3} and \ref{S5.3} that none of the adaptive schemes above guarantees both truthfulness and vanishing regret under selfish workers in our setting, motivating our new mechanism design.

\begin{figure}
    \centering
    \includegraphics[width=\linewidth]{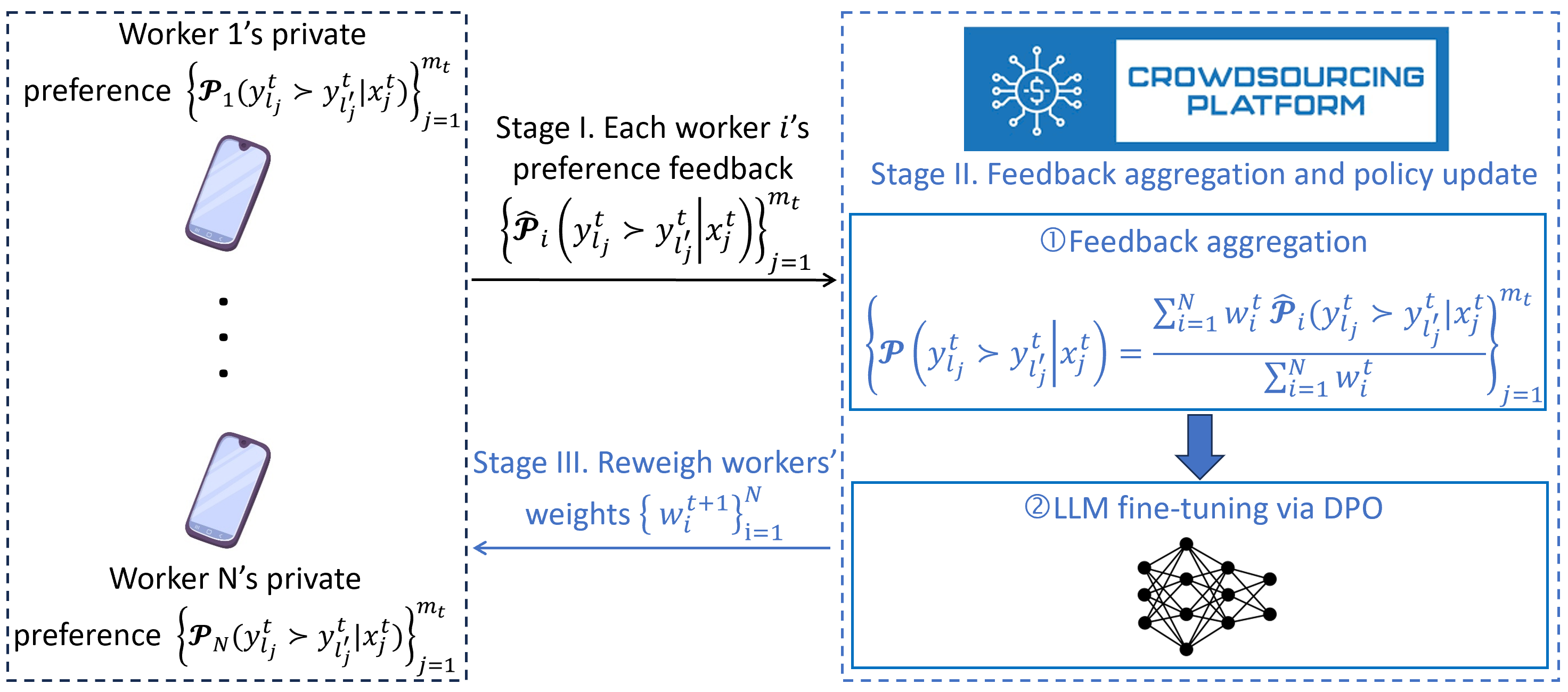}
    \caption{System model of LLM fine-tuning for mobile crowdsourcing. During each time slot $t \in [T]$, in Stage I, each mobile worker (e.g., mobile user) first reports his preference on the pairwise responses of each prompt (e.g., a traffic-state query in navigation or a band-availability query in spectrum sensing). In Stage II, the crowdsourcing platform (e.g., Waze) aggregates workers' feedback for fine-tuning the LLM and updates its policy. In Stage III, the platform adjusts each worker $i$'s weight $w_i^{t+1}$ according to his feedback accuracy for the next time slot $t+1$'s iteration.}
    \label{Fig1}
\end{figure}

\section{System Model and Problem Formulation}  \label{S2}

In Section~\ref{S2.1}, we introduce our system model. In Section~\ref{S2.2}, we formulate a new dynamic Bayesian game and give desired properties for guiding our late mechanism design.

\subsection{System Model of Online Mobile Crowdsourcing Based on LLM Fine-Tuning}\label{S2.1}

{LLM fine-tuning is increasingly used to enhance mobile applications by producing more accurate and verifiable outputs (e.g., \cite{karim2025large}). For example, navigation applications require route explanations that correctly reflect real-time traffic conditions (e.g., \cite{alsheeb2023towards}), while mobile spectrum-sensing systems rely on channel-usage interpretations that must align with the true occupancy state of the spectrum (e.g., \cite{aygul2025machine}). However, these mobile applications operate in highly dynamic environments, where traffic patterns and routing contexts can change minute by minute, and spectrum conditions and interference levels vary rapidly due to user mobility and network fluctuations. As a result, user preferences and feedback distributions evolve over time, causing offline fine-tuning to become quickly outdated. This dynamic nature necessitates online or iterative LLM fine-tuning, where the platform continuously incorporates newly collected mobile feedback to maintain reliable system performance under varying real-world conditions.}

Based on the above facts, we consider LLM fine-tuning for mobile crowdsourcing in which the crowdsourcing platform (e.g., Waze) receives preference feedback from $N \ge 2$ mobile workers (e.g., mobile users) over $T$ time slots (e.g., a weekly update cycle \cite{bai2022training, touvron2023llama, xiong2024iterative}). Each time slot $t \in [T]$ consists of the following three stages (as in Fig.~\ref{Fig1}).

\textit{1) Stage I. Online Worker Feedback}: 
The platform draws $m_t$ prompts $\{x_j^{t}\}_{j=1}^{m_t}$ from the context space $\mathcal{X}$ and a pair of candidate responses $\{(y_{l_j}^{t}, y_{l_j'}^{t} \mid x_j^{t})\}_{j=1}^{m_t}$ for each prompt from the response space $\mathcal{Y}$ (e.g., \cite{xiong2024iterative}). 
Each prompt corresponds to a task with a ground-truth system state. For navigation service, a prompt can be ``Does live traffic data indicate congestion on Route A between 5{:}20 PM and 5{:}40 PM?", where $y_{l_j}^{t}$ and $y_{l_j'}^{t}$ represent two alternative route explanations 
(e.g., one stating that congestion is present and another stating that conditions are clear).
For spectrum sensing, a prompt can be ``Do the measured signals indicate that the 3.5 GHz channel at (40.7, -74.0) is idle at time $t$?”, where $y_{l_j}^{t}$ and $y_{l_j'}^{t}$ represent two candidate channel-occupancy interpretations 
(e.g., one claiming the channel is idle and another claiming it is busy).
Such prompts have a correct answer determined by real-world conditions, so each pair of responses has a binary ground-truth system state.

The platform then shares $\{(x_j^t, y_{l_j}^{t}, y_{l_j'}^{t})\}_{j=1}^{m_t}$ with $N$ workers and collects their preference feedback. 
Each worker $i \in [N]$ forms a continuous private preference $\mathcal{P}_i(y_{l_j}^{t} \succ y_{l_j'}^{t} \mid x_j^t) \in [0,1],$ based on his own local observations (e.g., \cite{soumalias2024truthful,rafailov2024direct}). 
In navigation tasks, this belief may come from the worker's real-time mobile traffic view or recent travel experience. 
In spectrum-sensing tasks, it may come from locally measured signal strength or device-level sensing results. 
Worker $i$ holds the belief of the ground-truth system state as $p_j^t \sim \mathrm{Bernoulli}(\mathcal{P}_i(y_{l_j}^{t} \succ y_{l_j'}^{t} \mid x_j^t)),$
where realization $p_j^t = 1$ means response $y_{l_j}^{t}$ is preferred over $y_{l_j'}^{t}$ and $p_j^t = 0$ otherwise. {Note that such location-specific sensing context is distinctive to mobile crowdsourcing, where worker heterogeneity stems from spatial and temporal locality that the platform cannot observe directly. This motivates our Bayesian game formulation in Section~III-B.}

Aiming to increase his long-term influence or payment, each worker $i$ may report a continuous value $\hat{\mathcal{P}}_i(y_{l_j}^{t} \succ y_{l_j'}^{t} \mid x_j^t) \in [0,1]$
that differs from his true preference $\mathcal{P}_i(y_{l_j}^{t} \succ y_{l_j'}^{t} \mid x_j^t)$
(e.g., \cite{sun2024mechanism}, \cite{soumalias2024truthful}). 
 The platform and other workers are uncertain of his $\mathcal{P}_i(y_{l_j}^{t} \succ y_{l_j'}^{t}|x_j^t)$ realization.

\textit{2) Stage II. Online Feedback Aggregation and Policy Update}: After receiving each worker $i$'s reported preference values 
$\{\hat{\mathcal{P}}_i(y_{l_j}^{t}$$\succ$$ y_{l_j'}^{t}$$\mid$$x_j^t)\}_{j=1}^{m_t}$ for $i \in [N]$, 
the platform aggregates them using the current weight $w_i^{t}$ for each prompt $j \in [m_t]$:
\begin{align}\label{eq0}
    \mathcal{P}(y_{l_j}^{t} \succ y_{l_j'}^{t} \mid x_j^t)
= \frac{\sum_{i=1}^N w_i^{t} \hat{\mathcal{P}}_i(y_{l_j}^{t} \succ y_{l_j'}^{t} \mid x_j^t)}
{\sum_{i'=1}^N w_{i'}^{t}}.
\end{align}
All workers begin with a uniform weight $w_i^1 = 1$, which corresponds to the standard LLM fine-tuning practice of treating early feedback equally  
(e.g., \cite{rafailov2024direct,christiano2017deep,ouyang2022training}).

{The resulting aggregated preferences 
$\{\mathcal{P}(y_{l_j}^{t} \succ y_{l_j'}^{t} \mid x_j^t)\}_{j=1}^{m_t}$ 
directly form the human-feedback dataset 
$\mathcal{D}_t = \{(x_j^t, y_{l_j}^t, y_{l_j'}^t, \mathcal{P}(y_{l_j}^t \succ y_{l_j'}^t \mid x_j^t))\}_{j=1}^{m_t}$ 
for time slot $t$. 
Using this dataset, the platform updates the LLM policy via direct preference optimization (DPO), which solves the KL-regularized objective against a reference policy $\pi_{\texttt{ref}}$ as follows  \cite{xiong2024iterative,rafailov2024direct}:
\begin{align}\label{eq-dpo}
    \min_{\pi_t} -\mathbb{E}_{\mathcal{D}_t} \ln \sigma\bigg(\beta \ln \frac{\pi_t(y|x)}{\pi_{\texttt{ref}}(y|x)} - \beta \ln \frac{\pi_t(y'|x)}{\pi_{\texttt{ref}}(y'|x)}\bigg),
\end{align}
where $\sigma(\cdot)$ denotes the logistic function and $\beta$ is a parameter controlling the deviation from $\pi_{\texttt{ref}}$. DPO is a reward-model-free method that optimizes the policy directly from preference data without learning an explicit reward model. Our aggregation mechanism therefore influences the policy optimization step by determining the preference labels in $\mathcal{D}_t$, while leaving the underlying DPO update procedure unchanged.} {To reduce the platform-side fine-tuning cost, we adopt parameter-efficient LoRA fine-tuning (e.g., \cite{hu2022lora}), which keeps the base LLM weights frozen and only updates low-rank adapter matrices inserted into the attention projection layers.}

\textit{3) Stage III. Reweighing Workers}: After deploying the updated policy, the platform verifies the ground-truth system state $p_j^t \in \{0, 1\}$ for each prompt $j \in [m_t]$.\footnote{{Verification becomes available only at the 
end of each update cycle, so the delay does not block the online 
learning loop. Our mechanism also degrades gracefully under bounded 
noisy verification of the ground-truth system state, as formalized by Proposition~\ref{prop-robust} in Section~\ref{S4}.}} {In navigation applications, the platform verifies traffic conditions using authoritative infrastructure-side data that becomes available only after a delay, such as the California DOT's Performance Measurement System (PeMS) \cite{chen2001freeway}, which collects flow and occupancy measurements from nearly $40{,}000$ physical induction loop detectors embedded in the pavement. PeMS data are operated by the state DOT and released to third-party platforms only through delayed feeds, so they serve only as a post-hoc verification signal for worker reweighing rather than as a real-time substitute for worker reports. In spectrum-sensing applications, the fusion center likewise cannot observe primary-user activity in real time but can verify it afterwards by decoding ACK/NACK packets \cite{li2013sharp, liu2010indexability, zhao2005decentralized} or accessing post-transmission spectrum access system databases \cite{fcc_cbrs_2015, marshall2017three, winnf_sas_2017}.}
{Note that such infrastructure-side verification is distinctive to mobile crowdsourcing and unavailable in generic crowdsourcing tasks (e.g., MTurk). Since it arrives only after a delay, it serves as a post-hoc signal for worker reweighing in \eqref{e4} rather than a real-time substitute for worker reports.}

The platform then updates each worker $i$'s weight or payment $w_i^{t+1}$ for the next time slot:
\begin{align}\label{eq}
    w_i^{t+1}=f_i(\{\{\mathcal{\hat{P}}_i(y_{l_j}^{t} \succ y_{l_j'}^{t}|x_j^t)\}_{j=1}^{m_t}\}_{i=1}^N, \{p_j^t\}_{j=1}^{m_t})
\end{align}
according to his reported feedback $\{\{\mathcal{\hat{P}}_i(y_{l_j}^{t} $$\succ$$ y_{l_j'}^{t}|x_j^t)\}_{j=1}^{m_t}\}_{i=1}^N$ and  ground-truth system states $\{p_j^t\}_{j=1}^{m_t}$.

By strategically manipulating his reported preference 
$\hat{\mathcal{P}}_i(y_{l_j}^{t} \succ y_{l_j'}^{t} \mid x_j^t)$ (e.g., \cite{sun2024mechanism}, \cite{soumalias2024truthful}), each worker $i$ aims to maximize his long-term influence or payment from the platform, measured as his expected cumulative weight over $T$ time slots (e.g., \cite{zhang2012reputation},\cite{xie2016design}):
\begin{align}\label{e2}
&u_i(\{\{\mathcal{\hat{P}}_i(y_{l_j}^{t} \succ y_{l_j'}^{t}|x_j^t)\}_{j=1}^{m_t}\}_{t=1}^{T}) :=  \\ 
&\mathbb{E}\!\left[\sum_{t=1}^{T} w_i^{t}\bigg(\!\!\{\{\mathcal{\hat{P}}_i(y_{l_j}^{t-1} \!\!\succ\!\! y_{l_j'}^{t-1}|x_j^{t-1})\}_{j=1}^{m_{t-1}}\}_{i=1}^N,\! \{p_j^{t-1}\}_{j=1}^{m_{t-1}}\!\!\bigg)\right]. \nonumber
\end{align}
Here, the expectation is taken over the ground-truth system states $p_j^t$. We adopt the natural utility of expected cumulative weight.\footnote{Our mechanisms remain truthful when \eqref{e2} uses any increasing $g_i(\cdot)$ of expected cumulative weight. Extending to path- or distribution-dependent utilities (e.g., $\mathbb{E}[\sum_t g_i(w_i^t)]$) is however inherently difficult because workers' reports affect both current outcomes and the next weight $w_i^{t+1}$, thereby changing future selection probabilities and the distribution of the entire weight trajectory. Truthfulness would then require ruling out deviations that profit by reshaping this trajectory/distribution, which is challenging in sequential stochastic mechanisms.} 

On the other hand, the platform's alignment loss based on its aggregation over $T$ time slots is given as follows:
\begin{align*}
   L= \sum_{t=1}^T\frac{1}{m_t}\sum_{j=1}^{m_t}\bigg(\sum_{i=1}^N  
\frac{w_i^{t}\mathcal{\hat{P}}_i(y_{l_j}^{t} \succ y_{l_j'}^{t}|x_j^t)}{\sum_{i'=1}^N w_{i'}^{t}} - p_{j}^t\bigg)^2,
\end{align*} 
which is defined as the mean square error (MSE) between the platform's weighted aggregation in \eqref{eq0} and the realized binary outcome. It wants to improve the feedback accuracy in the aggregation by assigning the largest weight to the most accurate worker at time $t$ and such assignment will change over time. As each worker gives feedback over $T$ time slots, the ideal choice for the platform is to commit to the worker $i^*$ incurring the least average feedback loss over $T$ time slots:
\begin{align*}
    i^* = \arg\min_{i \in [N]} \sum_{t=1}^T\frac{1}{m_t}\sum_{j=1}^{m_t}\big(\mathcal{P}_i(y_{l_j}^{t} \succ y_{l_j'}^{t}|x_j^t) - p_{j}^t\big)^2.
\end{align*}
However, the best worker is unknown in the online iteration. The platform then turns to reducing the feedback accuracy regret between online weighted aggregation and offline choice of the best worker in hindsight as follows:
\begin{align}
    R(T) :=& \sum_{t=1}^T\frac{1}{m_t}\sum_{j=1}^{m_t}\bigg(\sum_{i=1}^N  
\frac{w_i^{t}\mathcal{\hat{P}}_i(y_{l_j}^{t} \succ y_{l_j'}^{t}|x_j^t)}{\sum_{i'=1}^N w_{i'}^{t}} - p_{j}^t\bigg)^2 \nonumber \\
    &- \min_{i \in [N]} \sum_{t=1}^T\frac{1}{m_t}\sum_{j=1}^{m_t}\big(\mathcal{P}_i(y_{l_j}^{t} \succ y_{l_j'}^{t}|x_j^t) - p_{j}^t\big)^2. \label{e3}
\end{align}

{Although our LLM fine-tuning for mobile crowdsourcing and federated learning (FL) both involve distributed edge participants, they address fundamentally different problems and therefore require different mechanism-design objectives. FL needs to use all participants' data to train a global model and view all data authentic and useful (e.g., \cite{wei_truthfedban_2024, wu_fedab_2023, zhao_truthful_fl_2023}). In contrast, our objective is to identify the best worker with the most accurate feedback over time to assign him largest weight and discard those inaccurate ones with vanishing weights.}

Note that the worker's utility in \eqref{e2} may not align with the platform's objective in \eqref{e3}, leading to untruthful feedback for a large weight (e.g., \cite{freeman2020no}). For example, suppose that prompt number $m_t=1$ and time slot number $T=1$. The platform updates $w_i^{t+1}=1$ if $|\mathcal{\hat{P}}_i^t-p_j^t|\leq 0.2$, $w_i^{t+1}=0.5$ if $|\mathcal{\hat{P}}_i^t-p_j^t|\in(0.2, 0.5]$, and $w_i^{t+1}=0$ otherwise. A worker $i$ holding $\mathcal{P}_i^t=0.6$ obtains an expected weight of $0.3$ in total by truthfully reporting. However, he can obtain an expected weight of $0.6$ by misreporting any $\mathcal{\hat{P}}_i^t \geq 0.8$, increased from being honest.  Therefore, it is crucial for the platform to properly design the weight update function for any worker's truthful feedback and a small regret.

{\textit{Remark (Communication and computation 
overhead).} LLM fine-tuning is performed centrally at the platform; 
workers only upload scalar preference reports 
$\{\mathcal{\hat{P}}_i(y_{l_j}^{t} \succ y_{l_j'}^{t}|x_j^t)\}_{j=1}^{m_t}$ 
per slot. In typical mobile crowdsourcing settings, $m_t$ and $N$ 
are on the order of tens to a hundred. For $N = 50$ and $m_t = 20$, 
each worker's per-slot uplink is $4 m_t = 80$~bytes, totaling 
$\sim 4$~KB at the platform; no model parameters, adapters, or 
gradients are transmitted. The per-slot aggregation in \eqref{eq0} 
and weight update in \eqref{e4} introduced by our mechanism cost 
$(2N+1) m_t + 2N m_t + N \approx 4000$ floating-point operations. 
The dominant cost is the platform-side LoRA fine-tuning, which is 
independent of our mechanism and trains only a small fraction of the 
full LLM (e.g., $0.24\%$ for GPT-2 124M), keeping GPT-2 fine-tuning tractable on a single NVIDIA A100 GPU. Workers perform no on-device LLM computation.}

\subsection{Dynamic Bayesian Game Formulation for RLHF}\label{S2.2}

Based on our system model above,  we formulate the multi-agent online learning between the crowdsourcing platform and $N$ strategic workers as a new dynamic Bayesian game:
\begin{itemize}
    \item In Stage I of each time slot $t$$\in$$[T]$, each worker $i$ with its private preference $\{\mathcal{P}_i(y_{l_j}^{t}$$\succ$$y_{l_j'}^{t}|x_j^t)\}_{j=1}^{m_t}$ determines his feedback $\{\mathcal{\hat{P}}_i(y_{l_j}^{t}$$\succ$$y_{l_j'}^{t}|x_j^t)\}_{j=1}^{m_t}$ (may not be the truth) to maximize his utility in \eqref{e2}.
    \item In Stage III of each time slot $t$$\in$$[T]$, the platform updates each worker's weight $w_i^{t+1}$=$f_i(\{\{\mathcal{\hat{P}}_i(y_{l_j}^{t} \succ y_{l_j'}^{t}|x_j^t)\}_{j=1}^{m_t}\}_{i=1}^N, \{p_j^t\}_{j=1}^{m_t})$ for reducing regret in \eqref{e3}.
\end{itemize}
 
Note that there is no strategic decision for any worker or the platform
in Stage II. We need to carefully design an online aggregation mechanism for ensuring each worker's truthful feedback and a vanishing regret over time. We define the desired properties as below.

\begin{definition}[Truthfulness of Worker Feedback]\label{def1}
An online weighted aggregation mechanism $\mathcal{M}$ is truthful if each worker $i$$\in$$[N]$ obtains a larger long-term influence or payment in \eqref{e2} over the whole $T$ time slots through truthful feedback instead of misreporting in the meantime, i.e., 
\begin{align*}
    \mathbb{E} \bigg[\sum_{t=1}^{T} w_{i}^{t} \big(&\{\mathcal{P}_i(y_{l_j}^{t-1} \succ y_{l_j'}^{t-1}|x_j^{t-1})\}_{j=1}^{m_{t-1}}, \{p_{j}^{t-1}\}_{j=1}^{m_{t-1}}, \\
    &\{\{\mathcal{\hat{P}}_k(y_{l_j}^{t-1} \succ y_{l_j'}^{t-1}|x_j^{t-1})\}_{j=1}^{m_{t-1}}\}_{k=1, k\ne i}^N\big)\bigg] \\
    \geq \mathbb{E} \bigg[\sum_{t=1}^{T} w_{i}^{t} \big(&\{\mathcal{\hat{P}}_i(y_{l_j}^{t-1} \succ y_{l_j'}^{t-1}|x_j^{t-1})\}_{j=1}^{m_{t-1}}, \{p_{j}^{t-1}\}_{j=1}^{m_{t-1}}, \\
   & \{\{\mathcal{\hat{P}}_k(y_{l_j}^{t-1} \succ y_{l_j'}^{t-1}|x_j^{t-1})\}_{j=1}^{m_{t-1}}\}_{k=1, k\ne i}^N\big)\bigg].
\end{align*}
\end{definition}

\begin{definition}[High Efficiency in Sublinear Regret $R(T)$ in \eqref{e3}]\label{def2}
    An online weighted aggregation mechanism $\mathcal{M}$ is efficient if its time-average regret $R_\mathcal{M}(T)/T$ is vanishing in the time slot number $T$, i.e., $\lim_{T\to\infty} 
\frac{R_\mathcal{M}(T)}{T} = 0.$
\end{definition}

\section{Benchmark Schemes}\label{S3}

In this section, we analyze three recent adaptive weighting schemes from the crowdsourcing and algorithmic game theory literature, serving as fair benchmarks to compare against later.

\subsection{Benchmark 1: EM-based Weight Estimation Scheme}

In the crowdsourcing literature, EM-based weight estimation (e.g., \cite{perez2021online, perez2018adaptive}) treats the true outcome $p_j^t$ as a hidden binary variable and iteratively estimates both the worker weights $w_i^t$ and the most likely outcome via Expectation-Maximization. We consider EM instantiations commonly used in crowdsourcing truth inference (e.g., Gaussian or Dawid--Skene--type models), where the E-step yields a monotone, majority-consistent estimate of the latent variable and the M-step rewards proximity to that estimate. Unfortunately, this scheme is not truthful and incurs a non-vanishing regret.

\begin{lemma}\label{L0}
   The benchmark 1 of EM-based weight estimation scheme is not truthful and incurs a regret in \eqref{e3} as $R_1(T)=\mathcal{O}(T)$, leading to a non-vanishing time-average regret $\lim_{T\to\infty}\frac{R_1(T)}{T}> 0$. 
\end{lemma}

The proof is given in Appendix~A of the supplementary material. Since EM relies on statistical consistency of worker reports, strategic workers can inflate their apparent reliability by aligning with the majority, earning disproportionate weights without accurate feedback.

\subsection{Benchmark 2: Hedge Scheme}

The Hedge scheme (e.g., \cite{fanuniversal, zhou2023efficient}) updates each worker $i$'s weight by exponential decay on the squared feedback loss:
\begin{align*}
    w_i^{t+1} = w_i^t \cdot e^{-\eta \cdot \frac{1}{m_t} \sum_{j=1}^{m_t}  (\mathcal{\hat{P}}_i(y_{l_j}^{t} \succ y_{l_j'}^{t}|x_j^t) - p_j^t)^2}, 
\end{align*}
for $i \in [N]$ and $t \in [T]$, where $\eta$ is the learning rate. Unfortunately, Hedge is not truthful as shown below.

\begin{lemma}\label{L1}
   Benchmark 2 of Hedge scheme is not truthful. 
\end{lemma}

The proof is given in Appendix~B of the supplementary material. As a worker with a large private preference can earn a higher weight by exaggerating his belief rather than reporting truthfully, Hedge fails to guarantee his truthful reporting.

\subsection{Benchmark 3: Median Aggregation Scheme}\label{S3.2}
In the algorithmic game theory literature, the ``median" scheme is widely used to incentivize truthful reporting from selfish agents (e.g., \cite{conitzer2024social}, \cite{wang2024positive}).
\begin{definition}[Median Aggregation Scheme]\label{defm} 
The platform first re-organizes workers' feedback $\{\mathcal{\hat{P}}_i(y_{l_j}^{t} \succ y_{l_j'}^{t}|x_j^t)\}_{i=1}^N$ in an increasing order as $\mathcal{\hat{P}}_{k_1,j}^t \leq \cdots \leq \mathcal{\hat{P}}_{k_N,j}^t$ for each prompt $j \in [m_t]$ in each time slot $t \in [T]$. It then chooses the median $\mathcal{\hat{P}}_{k_s,j}^t$ as its aggregation, where the index $s = \frac{N}{2}$ if $N$ is even and $s=\frac{N+1}{2}$ otherwise.
\end{definition}
Yet, this scheme still incurs non-vanishing regret.
\begin{lemma}\label{L1-1}
 The platform's regret in \eqref{e3} under the benchmark 3 of the median scheme is $R_3(T) = \mathcal{O}(T)$, leading to a non-vanishing time-average regret $\lim_{T\to\infty}\frac{R_3(T)}{T}>0$.  
\end{lemma}

The proof is given in Appendix~C of the supplementary material. Even under truthful reporting, the median fails to fully weight the most accurate worker $o$ with $\mathcal{P}_o(y_{l_j}^{t} \succ y_{l_j'}^{t}|x_j^t) = p_j^t$, yielding an $\mathcal{O}(T)$ aggregation loss while the best worker in hindsight incurs zero loss. These non-vanishing regrets across all three benchmarks motivate our truthful mechanism below.

\section{Truthful Online Weighted Aggregation Mechanism Design and Analysis}\label{S4}

As benchmarks 1-3 with untruthful worker feedback fail to identify the most accurate worker over time, we are well motivated to incentivize each worker's truthfulness and dynamically adjust each worker's weight according to his feedback accuracy in each time slot. In Stage III of each time slot, we assign a larger weight (compared to the others) if a worker's prior feedback is closer to the realized binary outcome. We need to carefully design our online mechanism weightage in \eqref{eq} to ensure that each obtains the largest long-term reputation in Definition~\ref{def1} only with truthful feedback. We define it in the following.  
\begin{definition}[Online Weighted Aggregation Mechanism]\label{def3}
At Stage III of each time slot $t \in [T]$, the platform updates each worker's weight $w_i^{t+1}$ in \eqref{eq} based on his feedback $\mathcal{\hat{P}}_i(y_{l_j}^{t} \succ y_{l_j'}^{t}|x_j^t)$ and the realized binary outcome $p_j^t$: 
\begin{align}\label{e4}
    w_{i}^{t+1} = w_{i}^{t} \cdot \bigg(1 - \alpha\frac{1}{m_t}\sum_{j=1}^{m_t} \big(\mathcal{\hat{P}}_i(y_{l_j}^{t} \succ y_{l_j'}^{t}|x_j^t)- p_j^t\big)^2 \bigg),
\end{align}
where $\alpha > 0$ is the step-size parameter to be determined later.\footnote{{Although $\alpha$ is uniform across workers, our mechanism remains responsive to new high-quality workers, as formalized by Proposition~\ref{prop-responsive} in Section~\ref{S4}.}}
\end{definition}

Intuitively, our mechanism determines each worker's weight in time slot $t+1$ based on his  feedback accuracy in the previous time slot $t$. If the squared difference between his feedback and the realized binary outcome $(\mathcal{\hat{P}}_i(y_{l_j}^{t} \succ y_{l_j'}^{t}|x_j^t)- p_j^t)^2$ is small, his weight $w_i^{t+1}$ will be only reduced by a small value from $w_i^{t}$. Though all workers' weights are decreasing over time, we care about the relative weighted aggregation as in \eqref{eq0} and the worker with a small decrement has a large influence in the platform's aggregation.  Our mechanism satisfies the truthful property as shown below.
\begin{proposition}\label{L2}
    Our mechanism in Definition~\ref{def3} is truthful, i.e., $\mathcal{\hat{P}}_i^*(y_{l_j}^{t} \succ y_{l_j'}^{t}|x_j^t) = \mathcal{P}_i(y_{l_j}^{t} \succ y_{l_j'}^{t}|x_j^t)$ for any prompt $j \in [m_t]$, SU $i \in [N]$ and time slot $t \in [T]$.
\end{proposition}

The proof is given in Appendix~D of the supplementary material of this TMC submission. As each worker holds a Bernoulli belief on outcome $p_j^t$, any deviation from truthful feedback leads to a strictly lower weight in any time slot $t$. Thus, no worker has the incentive to misreport and their truthfulness is guaranteed. Further, our mechanism is efficient and incurs a vanishing time-average regret in $T$.
\begin{theorem}\label{Thm1}
    Our online weighted aggregation mechanism in Definition~\ref{def3} incurs a sublinear regret $R_{\mathcal{M}}(T) $=$\mathcal{O}(T^{\frac{1}{2}})$ by choosing the step size $\alpha$ in \eqref{e4} as $$\alpha =\frac{2}{3}\sqrt{\frac{2\ln N}{T}},$$ leading to zero time-average regret with $\lim_{T\to\infty} 
\frac{R_\mathcal{M}(T)}{T} = 0$.
\end{theorem}
\textit{Proof.} According to Proposition~\ref{L2}, we have $\mathcal{\hat{P}}_i(y_{l_j}^t \succ y_{l_j'}^t|x_j^t) = \mathcal{P}_i(y_{l_j}^t \succ y_{l_j'}^t|x_j^t)$ for all $j \in [m_t]$, $i \in [N]$ and $t \in [T]$. To derive a lower bound on $\ln\frac{\sum_{i=1}^N w_{i}^{T+1}}{\sum_{i=1}^N w_{i}^1}$, we have
\begin{align}
  &\ln\frac{\sum_{i=1}^N w_{i}^{T+1}}{\sum_{i=1}^N w_{i}^1} \nonumber \\
  =& \ln\bigg( \sum_{i=1}^N w_{i}^{T+1} \bigg) - \ln\bigg( \sum_{i=1}^N w_{i}^1 \bigg) \nonumber \\
  =& \ln \bigg( \sum_{i=1}^N \prod_{t=1}^T ( 1-\alpha\frac{1}{m_t}\sum_{j=1}^{m_t} (\mathcal{P}_i(y_{l_j}^t \succ y_{l_j'}^t|x_j^t)- p_j^t)^2 ) \bigg) \!\!-\!\! \ln N \nonumber \\
  \geq& \ln \bigg( \prod_{t=1}^T ( 1-\alpha\frac{1}{m_t}\sum_{j=1}^{m_t} (\mathcal{P}_{i^*}(y_{l_j}^t \succ y_{l_j'}^t|x_j^t)- p_j^t)^2 )  \bigg) - \ln N \nonumber 
\end{align}
\begin{align}
  =& \sum_{t=1}^T \ln \bigg( 1-\alpha \frac{1}{m_t}\sum_{j=1}^{m_t}(\mathcal{P}_{i^*}(y_{l_j}^t \succ y_{l_j'}^t|x_j^t)- p_j^t)^2   \bigg) - \ln N \nonumber \\
  \geq& -\alpha\sum_{t=1}^T\frac{1}{m_t}\sum_{j=1}^{m_t} (\mathcal{P}_{i^*}(y_{l_j}^t \succ y_{l_j'}^t|x_j^t)- p_j^t)^2 \nonumber \\
  &-\alpha^2\sum_{t=1}^T\bigg(\frac{1}{m_t}\sum_{j=1}^{m_t} (\mathcal{P}_{i^*}(y_{l_j}^t \succ y_{l_j'}^t|x_j^t)- p_j^t)^2\bigg)^2 - \ln N \nonumber \\
  \geq& -\alpha\sum_{t=1}^T\frac{1}{m_t}\sum_{j=1}^{m_t} \bigg(\mathcal{P}_{i^*}(y_{l_j}^t \succ y_{l_j'}^t|x_j^t)- p_j^t\bigg)^2 -\alpha^2 T - \ln N, \label{low'} 
\end{align}
where we choose $\alpha < \frac{1}{2}$ and denote $i^*$ as the best worker in hindsight. The first and the third inequalities hold due to $0 < \alpha < \frac{1}{2}$ and $0 \leq (\mathcal{P}_{i^*}(y_{l_j}^t \succ y_{l_j'}^t|x_j^t)- p_j^t)^2 \leq 1$ for all $i \in [N]$ and $t \in [T]$. The second inequality holds due to $\ln (1 - x) \geq -x -x^2$ for $x \leq \frac{1}{2}$.

To derive an upper bound on $\ln\frac{\sum_{i=1}^N w_{i}^{t+1}}{\sum_{i=1}^N w_{i}^t}$, we have
\begin{align}
  &\ln\frac{\sum_{i=1}^N w_{i}^{t+1}}{\sum_{i=1}^N w_{i}^t} \nonumber \\
  =& \ln \bigg( \frac{\sum_{i=1}^N w_{i}^t \cdot (1-\alpha\frac{1}{m_t}\sum_{j=1}^{m_t} (\mathcal{P}_i(y_{l_j}^t \succ y_{l_j'}^t|x_j^t)- p_j^t)^2)}{\sum_{i'=1}^N w_{i'}^t} \bigg) \nonumber \\
  \leq& \ln \bigg( \frac{\sum_{i=1}^N w_{i}^t \cdot e^{-\alpha\frac{1}{m_t}\sum_{j=1}^{m_t} (\mathcal{P}_i(y_{l_j}^t \succ y_{l_j'}^t|x_j^t)- p_j^t)^2}}{\sum_{i'=1}^N w_{i'}^t} \bigg) \nonumber \\
  \leq& \!\! -\alpha \frac{1}{m_t}\sum_{j=1}^{m_t}\frac{\sum_{i=1}^N w_{i}^t(\mathcal{P}_i(y_{l_j}^t \succ y_{l_j'}^t|x_j^t)- p_j^t)^2}{\sum_{i'=1}^N w_{i'}^t}  + \frac{\alpha^2}{8},  \label{upp'} 
\end{align}
where the first inequality holds due to $1-\alpha x \leq e^{-\alpha x}$ for $0 \leq x \leq 1$ and $\alpha > 0$, the second due to Hoeffding's lemma: for a random variable $X = -\frac{1}{m_t}\sum_{j=1}^{m_t}(\mathcal{P}_i(y_{l_j}^t \succ y_{l_j'}^t|x_j^t)- p_j^t)^2  \in [-1, 0]$ and $\alpha \in R$, we have
\begin{align*}
    \ln (\mathbb{E}[e^{\alpha X}]) \leq \alpha \mathbb{E}[X] + \frac{\alpha^2(1-0)^2}{8}.
\end{align*}
According to \eqref{upp'}, we have
\begin{align}
    &\ln\frac{\sum_{i=1}^N w_{i}^{T+1}}{\sum_{i=1}^N w_{i}^1} \nonumber \\
   =  &\ln\bigg(\frac{\sum_{i=1}^N w_{i}^{T+1}}{\sum_{i=1}^N w_{i}^t} \frac{\sum_{i=1}^N w_{i}^t}{\sum_{i=1}^N w_{i}^{t-1}} \cdot\cdots\cdot\frac{\sum_{i=1}^N w_{i}^2}{\sum_{i=1}^N w_{i}^1}\bigg) \nonumber \\
    = &\sum_{t=1}^T \ln\frac{\sum_{i=1}^N w_{i}^{T+1}}{\sum_{i=1}^N w_{i}^t} \nonumber \\
    \leq& -\alpha\sum_{t=1}^T\frac{1}{m_t}\sum_{j=1}^{m_t} \frac{\sum_{i=1}^N w_{i}^t(\mathcal{P}_i(y_{l_j}^t \succ y_{l_j'}^t|x_j^t)- p_j^t)^2}{\sum_{i'=1}^N w_{i'}^t}  + \frac{\alpha^2 T}{8}. \label{uppp''}
\end{align}
According to \eqref{low'} and \eqref{uppp''}, we have
\begin{align*}
&-\alpha\sum_{t=1}^T\frac{1}{m_t}\sum_{j=1}^{m_t} \bigg(\mathcal{P}_{i^*}(y_{l_j}^t \succ y_{l_j'}^t|x_j^t)- p_j^t\bigg)^2 -\alpha^2 T - \ln N  
\end{align*}
\begin{align*}
\leq& -\alpha\sum_{t=1}^T\frac{1}{m_t}\sum_{j=1}^{m_t} \frac{\sum_{i=1}^N w_{i}^t(\mathcal{P}_i(y_{l_j}^t \succ y_{l_j'}^t|x_j^t)- p_j^t)^2}{\sum_{i'=1}^N w_{i'}^t}  + \frac{\alpha^2 T}{8}.
\end{align*}
After re-arranging the above inequalities and dividing $\alpha$ on both sides, we have
\begin{align*}
   &\sum_{t=1}^T\frac{1}{m_t}\sum_{j=1}^{m_t} \frac{\sum_{i=1}^N w_{i}^t(\mathcal{P}_i(y_{l_j}^t \succ y_{l_j'}^t|x_j^t)- p_j^t)^2}{\sum_{i'=1}^N w_{i'}^t} \\
   &- \sum_{t=1}^T\frac{1}{m_t}\sum_{j=1}^{m_t} \bigg(\mathcal{P}_{i^*}(y_{l_j}^t \succ y_{l_j'}^t|x_j^t)- p_j^t\bigg)^2 
   \leq \frac{\ln N}{\alpha} + \frac{9T\alpha}{8}.
\end{align*}
Choosing $\alpha = \frac{2}{3}\sqrt{\frac{2\ln N}{T}} < \frac{1}{2}$ (true as $T\to\infty$), we have
\begin{align*}
  &\sum_{t=1}^T\frac{1}{m_t}\sum_{j=1}^{m_t} \frac{\sum_{i=1}^N w_{i}^t(\mathcal{P}_i(y_{l_j}^t \succ y_{l_j'}^t|x_j^t)- p_j^t)^2}{\sum_{i'=1}^N w_{i'}^t} \\
  &- \sum_{t=1}^T\frac{1}{m_t}\sum_{j=1}^{m_t} \bigg(\mathcal{P}_{i^*}(y_{l_j}^t \succ y_{l_j'}^t|x_j^t)- p_j^t\bigg)^2 \leq 3\sqrt{\frac{T\ln N}{2}}.
\end{align*}
Finally, we have the regret $R_{\mathcal{M}}(T)$ satisfying
\begin{align*}
    R_{\mathcal{M}}(T) = &\sum_{t=1}^T\frac{1}{m_t}\sum_{j=1}^{m_t}\bigg(\sum_{i=1}^N \frac{w_i^{t}\mathcal{\hat{P}}_i(y_{l_j}^t \succ y_{l_j'}^t|x_j^t)}{\sum_{i'=1}^N w_{i'}^{t}} - p_j^t\bigg)^2 \\
    &- \sum_{t=1}^T\frac{1}{m_t}\sum_{j=1}^{m_t}\bigg(\mathcal{P}_{i^*}(y_{l_j}^t \succ y_{l_j'}^t|x_j^t) - p_j^t\bigg)^2 \\
    \leq& \sum_{t=1}^T\frac{1}{m_t}\sum_{j=1}^{m_t} \frac{\sum_{i=1}^N w_{i}^t(\mathcal{P}_i(y_{l_j}^t \succ y_{l_j'}^t|x_j^t)- p_j^t)^2}{\sum_{i'=1}^N w_{i'}^t} \\
    &- \sum_{t=1}^T\frac{1}{m_t}\sum_{j=1}^{m_t} \bigg(\mathcal{P}_{i^*}(y_{l_j}^t \succ y_{l_j'}^t|x_j^t)- p_j^t\bigg)^2 \\
    \leq&3\sqrt{\frac{T\ln N}{2}} = \mathcal{O}(T^{\frac{1}{2}}),
\end{align*}
where the first inequality holds due to the convexity of the aggregation loss function. We then finish the proof. $\hfill\square$

According to Theorem~\ref{Thm1}, our mechanism obviously improves from benchmarks 1-3 by distinguishing the most accurate worker in the online learning process as $T\to\infty$. As $N$ increases, the platform may find a more accurate worker in hindsight. Thus, it chooses a larger step-size $\alpha$ in \eqref{e4} to 
punish inaccurate workers more in the weighted aggregation to retire them. As $T$ increases, the platform is more patient in choosing a smaller $\alpha$ in \eqref{e4} to select the best worker in hindsight with more time slots and samples. 

{We further strengthen Theorem~\ref{Thm1} with two additional properties of our mechanism. Proposition~\ref{prop-responsive} below characterizes the responsiveness of our mechanism to new high-quality workers under the uniform step-size $\alpha$. Proposition~\ref{prop-robust} below establishes the robustness of our mechanism against noisy verification of the ground-truth system state.}

\begin{proposition}[{Responsiveness to a newly-arriving high-quality worker under uniform $\alpha$}]\label{prop-responsive}
{Consider an existing worker $k \in [N]$ active prior to slot $t_0 \in [T]$ with weight $w_k^{t_0} \in (0, \infty)$ accumulated through prior reweighing, and a newly-arriving worker $i$ entering the system at slot $t_0$ with initial weight $w_i^{t_0} \in (0, w_k^{t_0})$, so that the newly-arriving worker starts at a strict weight disadvantage $w_i^{t_0} < w_k^{t_0}$. Both workers continue to participate in the mechanism from slot $t_0$ onward under the weight update~\eqref{e4}. Suppose under truthful reporting, the newly-arriving worker is strictly more accurate, with expected per-slot squared losses
\begin{align*}
    \ell_i^t := \mathbb{E}\bigg[\frac{1}{m_t}\sum_{j=1}^{m_t}\big(\mathcal{P}_i(y_{l_j}^t \succ y_{l_j'}^t \mid x_j^t) - p_j^t\big)^2\bigg], 
\end{align*}
satisfying $\ell_i^t < \ell_k^t$ for all $t \ge t_0$ and accuracy gap $\Delta := \min_{t \ge t_0}(\ell_k^t - \ell_i^t) > 0$. Under our mechanism in Definition~\ref{def3} with step-size $\alpha = \frac{2}{3}\sqrt{2 \ln N/T}$ from Theorem~\ref{Thm1}, the newly-arriving worker overtakes the existing worker in expected weight within}
\begin{align*}
    {\tau_{\textnormal{new}} \le \bigg\lceil \frac{\ln(w_k^{t_0}/w_i^{t_0})}{\alpha \Delta}\bigg\rceil = \mathcal{O}\bigg(\sqrt{\frac{T}{\ln N}} \cdot \Delta^{-1}\bigg)}
\end{align*}
{slots after arrival, regardless of how the existing worker accumulated its weight prior to slot $t_0$.}
\end{proposition}

{The proof is given in Appendix~H of the supplementary material of this TMC submission. Proposition~\ref{prop-responsive} shows that our mechanism adapts quickly to a newly-arriving high-quality worker even when an existing worker enters the comparison with a substantially higher accumulated weight. The bound's numerator $\ln(w_k^{t_0}/w_i^{t_0})$ scales only logarithmically with the weight asymmetry, so even a large prior weight advantage of the existing worker is overcome in $\mathcal{O}(\sqrt{T/\ln N} \cdot \Delta^{-1})$ slots, independent of the existing worker's tenure prior to $t_0$. By symmetry, the same bound applies whenever any worker becomes more accurate than a previously dominant worker, regardless of how the dominant worker accumulated its weight from prior reputation or earlier accurate reporting. This responsiveness is empirically verified in Figs.~\ref{fig1} and \ref{fig3}, where the most accurate worker's chosen probability rises from the uniform initialization of $1/N$ to over $0.9$ (full feedback) and near $0.8$ (limited feedback) within the time horizon.}

\begin{proposition}[{Robustness under noisy verification}]\label{prop-robust}
{Suppose the verified ground-truth label $\tilde{p}_j^t$ differs from the true $p_j^t$ with probability at most $\epsilon \in [0, 1/2)$, independently across prompts and independently of all workers' reports. Under our mechanism in Definition~\ref{def3} with step-size $\alpha = \frac{2}{3}\sqrt{2 \ln N/T}$ from Theorem~\ref{Thm1}, the following hold:
\begin{enumerate}
    \item[(a)] (\textit{Truthfulness degradation.}) The best-response misreport $\hat{\mathcal{P}}_i^*$ of any worker $i$ deviates from his true preference $\mathcal{P}_i$ by at most $\epsilon$, i.e., $|\hat{\mathcal{P}}_i^* - \mathcal{P}_i| \le \epsilon$. The cumulative strategic gain from misreporting over $T$ slots is at most $\mathcal{O}(\epsilon^2 \sqrt{T})$, which is dominated by the $\mathcal{O}(\sqrt{T})$ regret term and vanishes as $\epsilon \to 0$.
\item[(b)] (\textit{Regret degradation.}) The expected time-average regret satisfies
    \begin{align*}
        \frac{\mathbb{E}[R_{\mathcal{M}}(T)]}{T} \le \mathcal{O}\bigg(\frac{1}{\sqrt{T}}\bigg) + 2\epsilon,
    \end{align*}
    recovering the clean-case bound in Theorem~\ref{Thm1} as $\epsilon \to 0$.
\end{enumerate}}
\end{proposition}

{The proof is given in Appendix~I of the supplementary material of this TMC submission. Proposition~\ref{prop-robust} extends Theorem~\ref{Thm1} to settings where the infrastructure-side verification is imperfect, such as when decoded ACK/NACK packets in spectrum sensing are occasionally corrupted or when PeMS traffic flow measurements in navigation exhibit occasional sensor faults. The analogous result for the limited-feedback setting in Theorem~\ref{Thm1-} can be derived similarly, with an additive $\mathcal{O}(\epsilon)$ term added to the time-average regret.}

\section{Extension to Limited Worker Feedback}\label{S5}

Recall that in Sections~\ref{S2}-\ref{S4}, we assume that the platform has access to all the workers' feedback per time slot. {In practice, collecting feedback from multiple workers can be difficult due to cost and coordination challenges unique to mobile devices (such as battery, sensing, and uplink constraints that are absent in desktop-based crowdsourcing platforms), which can in turn slow down the iterative LLM fine-tuning (e.g., \cite{lambert2025reinforcement}).} For example, in navigation applications, querying many workers about real-time traffic events (e.g., whether a reported accident or congestion is actually present) increases monetary cost and delays the update cycle. 
Similarly, in spectrum-sensing applications, asking multiple mobile users to label channel conditions (e.g., whether a 3.5 GHz channel is idle or occupied at a specific time and location) may require additional sensing operations and energy consumption, which limits the number of workers that can provide feedback. 

In this section, we extend to consider a challenging case where the platform can receive only one worker's report transmission per time slot. In the following, we first present our system model for this limited worker feedback scenario and the dynamic Bayesian game formulation. We then give our mechanism design and analysis.

\subsection{System Model of Limited Worker Feedback}

Similar to the system model in Section~\ref{S2.1}, the platform iterates the online learning process in $T$ time slots, where each time slot $t$$\in$$[T]$ still contains three stages. In Stage I, instead of querying all the $N$ workers' feedback, the platform can only select one worker $I_t \in [N]$ in each time slot $t$ for his local observations and preference feedback. We consider that the platform uses a mixed strategy to select each worker $i$ with a probability of $\frac{w_i^t}{\sum_{i'\in[N]}w_{i'}^t}$ in each time slot $t$.  In Stage II, after receiving the chosen worker's  feedback $\{\mathcal{\hat{P}}_{I_t}(y_{l_j}^{t} \succ y_{l_j'}^{t}|x_j^t)\}_{j=1}^{m_t}$, the platform determines the preference for each prompt $j \in[m_t]$ as follows:
    \begin{align}
        \mathcal{P}(y_{l_j}^{t} \succ y_{l_j'}^{t}|x_j^t)=  \mathcal{\hat{P}}_{I_t}(y_{l_j}^{t} \succ y_{l_j'}^{t}|x_j^t),\label{eq0-}
\end{align}
which will be included to construct the human-annotation dataset $\{\mathcal{P}(y_{l_j}^{t} \succ y_{l_j'}^{t}|x_j^t)\}_{j=1}^{m_t}$ for training and updating the LLM fine-tuning policy later. In Stage III, the platform 
dynamically adjusts each worker's weight in the online learning process and determines  
\begin{align}\label{eq-}
    w_i^{t+1}=f_i(\{\mathcal{\hat{P}}_{I_t}(y_{l_j}^{t} \succ y_{l_j'}^{t}|x_j^t)\}_{j=1}^{m_t}, \{p_j^t\}_{j=1}^{m_t})
\end{align}
for the next slot $t+1$'s selection according to  feedback $\{\mathcal{\hat{P}}_{I_t}(y_{l_j}^{t} \succ y_{l_j'}^{t}|x_j^t)\}_{j=1}^{m_t}$ and the realized ground truth physical state $\{p_j^t\}_{j=1}^{m_t}$, where $w_i^1=1$ for any $i\in[N]$. 

By strategically manipulating his reported preference 
$\hat{\mathcal{P}}_i(y_{l_j}^{t} \succ y_{l_j'}^{t} \mid x_j^t)$, each worker $i$ aims to maximize his long-term reputation or payment from the platform over whole $T$ time slots as in \eqref{e2}.

On the other hand, the platform adopts a mixed strategy to choose each worker $i\in[N]$ with a probability of $\frac{w_i^t}{\sum_{i'\in[N]}w_{i'}^t}$ in each time slot $t$. It wants to improve the feedback accuracy in the aggregation by assigning the largest weight to the most accurate worker. As the best worker is still unknown in the online iteration, it then aims to reduce the regret between online mixed selection and offline choice of the best worker in hindsight, where the alignment loss is defined as the MSE between the platform's mixed selection and the realized binary ground truth physical state as follows:
\begin{align}
    R(T) := &\sum_{t=1}^T\sum_{i=1}^N \frac{w_i^t}{\sum_{i'=1}^N w_{i'}^t}  \frac{1}{m_t}\sum_{j=1}^{m_t}\bigg(\mathcal{\hat{P}}_{i}(y_{l_j}^{t} \succ y_{l_j'}^{t}|x_j^t) - p_{j}^t\bigg)^2 \nonumber \\
    &- \min_{i \in [N]} \sum_{t=1}^T\frac{1}{m_t}\sum_{j=1}^{m_t}\big(\mathcal{P}_i(y_{l_j}^{t} \succ y_{l_j'}^{t}|x_j^t) - p_{j}^t\big)^2, \nonumber \\
    := &\sum_{t=1}^T\sum_{i=1}^N \frac{w_i^t}{\sum_{i'=1}^N w_{i'}^t} \hat{\ell}_i^t - \min_{i \in [N]} \sum_{t=1}^T\ell_i^t,
    \label{e3-}
\end{align} 
where we define $\hat{\ell}_i^t := \frac{1}{m_t}\sum_{j=1}^{m_t}\big(\mathcal{\hat{P}}_i(y_{l_j}^{t} \succ y_{l_j'}^{t}|x_j^t) - p_{j}^t\big)^2$ and $\ell_i^t := \frac{1}{m_t}\sum_{j=1}^{m_t}\big(\mathcal{P}_i(y_{l_j}^{t} \succ y_{l_j'}^{t}|x_j^t) - p_{j}^t\big)^2$.

\subsection{Dynamic Bayesian Game Formulation under Limited Worker Feedback}\label{S5.2}

Based on our system model above,  we formulate the multi-agent online learning as a new dynamic Bayesian game:
\begin{itemize}  
    \item In Stage I of each time slot $t$$\in$$[T]$, the platform first chooses a worker $I_t\in[N]$ for his feedback. Then, the chosen worker $I_t$ with his private preference $\{\mathcal{P}_{I_t}(y_{l_j}^{t}$$\succ$$y_{l_j'}^{t}|x_j^t)\}_{j=1}^{m_t}$  determines his feedback $\{\mathcal{\hat{P}}_{I_t}(y_{l_j}^{t}$$\succ$$y_{l_j'}^{t}|x_j^t)\}_{j=1}^{m_t}$ to maximize his utility in \eqref{e2}.
    \item In Stage III of each time slot $t$$\in$$[T]$, the platform updates each worker's weight $w_i^{t+1}$=$f_i(\{\mathcal{\hat{P}}_{I_t}(y_{l_j}^{t} \succ y_{l_j'}^{t}|x_j^t)\}_{j=1}^{m_t}, \{p_j^t\}_{j=1}^{m_t})$  for reducing its regret in \eqref{e3-}.
\end{itemize}

Note that there is no strategic decision for any worker or the platform in Stage II. We need to carefully design an online mixed selection mechanism for ensuring each worker's truthful feedback and a vanishing regret over time as given in Definitions~\ref{def1} and \ref{def2}. Before that, we introduce the EXP3 scheme in the crowdsourcing literature to check.

\subsection{Benchmark 4: EXP3 for Limited Worker Feedback}\label{S5.3}

The EXP3 scheme as a variant of Hedge in the crowdsourcing literature (e.g., \cite{kash2024slowly}, \cite{khodak2023meta}) for limited worker feedback updates worker $i$'s weight as follows:
\begin{align*}
    w_i^{t+1} = \begin{cases}
        w_i^{t} \cdot e^{-\eta \cdot \tilde{\ell}_i^t}, &\text{if} \ i = I_t, \\
        w_i^t, &\text{otherwise,}
    \end{cases} 
\end{align*}
where an unbiased estimator $\tilde{\ell}_i^t$ of the worker $i$'s feedback loss is given as follows:
\begin{align}\label{e-ue}
  \tilde{\ell}_i^t =  \begin{cases}
        \frac{\ell_i^t}{(1-\beta)w_i^t/\sum_{i'=1}^N w_{i'}^t+\beta/N}, &\text{if} \ i = I_t, \\
        0, &\text{otherwise.}
    \end{cases}
\end{align}
Unfortunately, such an EXP3 scheme is still not truthful as shown below.
\begin{lemma}\label{Benchmark 4}
   The benchmark 4 of EXP3 scheme is not truthful.
\end{lemma}

The proof is given in Appendix~E of the supplementary material of this TMC submission. Lemma~\ref{Benchmark 4} indicates that it is non-trivial to incentivize workers' truthful reports under the limited worker feedback case. Therefore, a new weight update scheme needs to be investigated for limited worker feedback.

\subsection{Our Truthful Online Mixed Selection Mechanism Design and Analysis}

As the platform can only observe one worker's feedback in each time slot, it needs to carefully balance the exploitation of selected workers and the exploration of unselected ones during the online learning process. Further, with such limited feedback information, it is even more challenging to guarantee a sublinear regret. The idea of our new mechanism design is to introduce an exploitation parameter and update the selected worker's weight with some probability. We define our mechanism in the following. 
\begin{definition}[Online Mixed Selection Mechanism]\label{def3-}
At Stage III of each time slot $t \in [T]$, the platform updates each 
worker's weight $w_i^{t+1}$ in \eqref{eq-} as follows: 
\begin{align}\label{e4-}
    w_i^{t+1} = \begin{cases}
        (1-\beta) \gamma_i^{t+1} +  \beta, &\text{if} \ i = I_t, \\
        w_i^t, &\text{otherwise,}
    \end{cases} 
\end{align}
where $\beta \in (0, 1)$ is an exploitation parameter to determine later, 
\begin{align}\label{w}
    \gamma_i^{t+1} = \begin{cases}
         \gamma_i^{t} \bigg(1 - \alpha \frac{\ell_i^{t}(1-\alpha/\theta_{i}^t)}{\theta_i^t}\bigg), &\text{if} \ i = I_t, \\
         \gamma_i^t, & \text{otherwise},
    \end{cases}
\end{align}
$\theta_i^t = \frac{w_i^t}{\sum_{i'=1}^N w_{i'}^t}$ denotes the probability of selecting worker $i$, $\gamma_i^1 = 1$ for $i \in [N]$, and $\alpha \in (0, \theta_i^t)$ is a step-size parameter to determine later.
\end{definition}

Our mechanism in Definition~\ref{def3-} balances the exploration and the exploitation during the online learning process. It updates each chosen worker's weight $w_i^{t+1}$ in \eqref{e4-} only with a probability of $1-\beta$ for uniform exploration. Note that if a worker is frequently chosen before, his weight keeps decreasing from 1 and becomes smaller than the others. Therefore, the probability that he will be chosen in future time slots is lower than that of the unchosen workers, especially when his feedback accuracy is low. Further, a chosen worker $i$'s weight $w_i^{t+1}$ in \eqref{e4-} is proportional to $\gamma_i^{t+1}$ in \eqref{w}, which will be decreased by a small value if his feedback loss $\hat{\ell}_i^t$ is small, otherwise large. By carefully designing $\alpha$ and $\beta$, our mechanism can guarantee valid $w_i^t$ in \eqref{e4-} and $\gamma_i^t$ in \eqref{w}.

Since each worker holds a Bernoulli belief on the ground truth $p_j^t$, our mechanism satisfies the truthful property as shown below.
\begin{proposition}\label{L2-}
    Our mechanism in Definition~\ref{def3-} is truthful, i.e., $\mathcal{\hat{P}}_i^*(y_{l_j}^{t} \succ y_{l_j'}^{t}|x_j^t) = \mathcal{P}_i(y_{l_j}^{t} \succ y_{l_j'}^{t}|x_j^t)$ for any prompt $j \in [m_t]$, SU $i \in [N]$ and time slot $t \in [T]$.
\end{proposition}

\begin{figure}
    \centering   \includegraphics[width=0.9\linewidth]{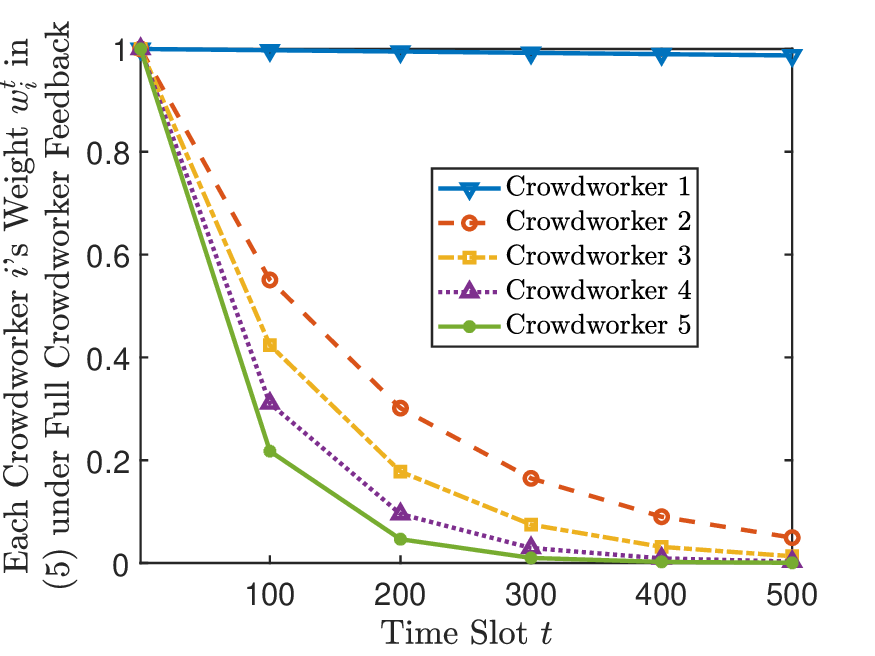}
    \caption{Each worker $i$'s weight $w_i^t$ of our truthful online weighted aggregation versus time slot $t$ under full feedback. Here we fix worker number $N=5$, total time slot number $T=500$, and prompt number $m_t=20$.} 
    \label{fig0}
\end{figure}

The proof is given in Appendix~F of the supplementary material of this TMC submission. Further, our mechanism is efficient and incurs a vanishing time-average regret in $T$.
\begin{theorem}\label{Thm1-}
    Our mechanism in Definition~\ref{def3-} incurs the sublinear accuracy regret $R_{\mathcal{M}}(T) $=$\mathcal{O}(\sqrt{T})$ by choosing step-sizes $\beta$ in \eqref{e4-} and $\alpha$ in \eqref{w} as follows:
    \begin{align*}
        \beta = 2\sqrt{\frac{N\ln N}{7T}}, \ \alpha = \sqrt{\frac{\ln N}{7NT}},
    \end{align*} leading to zero time-average regret with $\lim_{T\to\infty} 
\frac{R_\mathcal{M}(T)}{T} = 0$.
\end{theorem}

The proof is given in Appendix~G of the supplementary material of this TMC submission. According to Theorem~\ref{Thm1-}, our mechanism can distinguish the most accurate worker in the online process as $T\to\infty$. As the worker number $N$ increases, the platform faces with more uncertainty of feedback accuracy in exploration than exploitation. Thus, it becomes more patient with a larger exploitation parameter $\beta$ and a smaller step-size $\alpha$ to punish inaccurate (chosen) workers less in the weight update. On the other hand, as the time slot number $T$ increases, the platform has more room to explore for the most accurate worker with more time slots and samples. Thus, it chooses a smaller exploitation parameter $\beta$ and a smaller step-size $\alpha$ to punish inaccurate (chosen) workers less in the weight update.

\section{Experiments}\label{S6}

In this section, we run experiments to show our mechanism's great improvement over the benchmark schemes. In Section~\ref{S6.1}, we use synthetic data to evaluate the performance of our proposed mechanisms against benchmarks. In Sections~\ref{S6.2}, we further fine-tune LLMs on real-world datasets to validate our mechanisms' great advantages.

\subsection{Experiments on Synthetic Data}\label{S6.1}

For ease of exposition and illustration, we first consider $N$=5 workers and fix prompt number $m_t$=20. We randomly generate the binary realized ground truth $p_j^t \in \{1, 0\}$ for each prompt $j$$\in$$[m_t]$. We add random noise from the set of $\{[0, 0.1], [0.45, 0.55], [0.55, 0.65], [0.65, 0.75], [0.75, 0.85]\}$ to workers' preferences in order. For example, worker 1 is the most accurate with his preference noise in the range of [0, 0.1] on the realized ground truth, and worker 5 is the least accurate with his preference noise in the range of [0.75, 0.85] on the realized binary ground truth. 

Figure~\ref{fig0} shows how each worker $i$'s weight $w_i^t$ evolves over time slot $t$$\in$$ [T]$ under full worker feedback, where each worker's feedback is available to the platform in each time slot $t$. As $t$ increases to 500, our mechanism manages to assign the largest weight to the most accurate worker 1 and assign much smaller weights to the remaining ones especially for the most inaccurate worker 5, which is consistent with each worker's weight \eqref{e4} in Definition~\ref{def3}. 

\begin{figure}
    \centering   \includegraphics[width=0.9\linewidth]{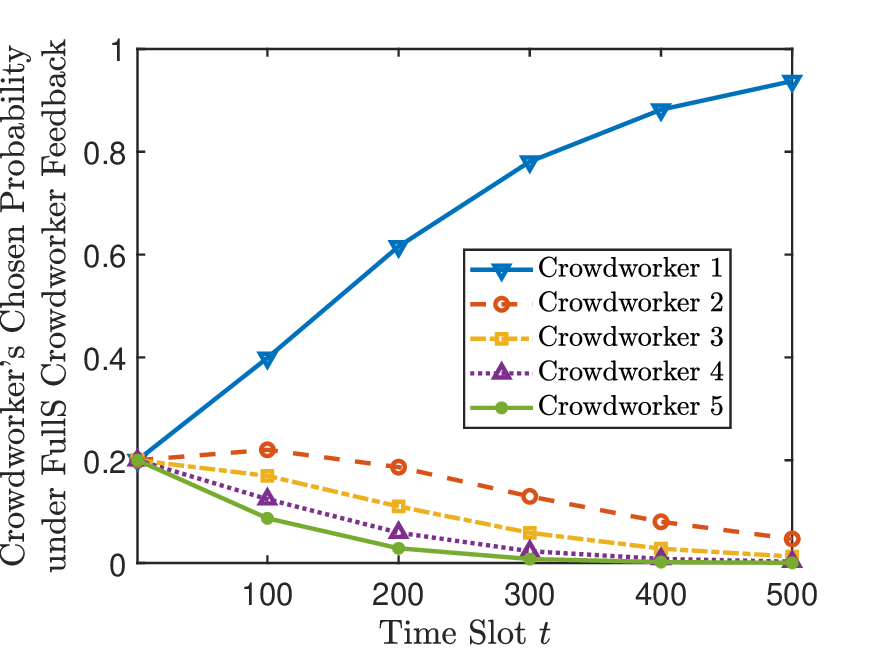}
    \caption{Each worker $i$'s chosen probability $\frac{w_i^t}{\sum_{i'=1}^N w_{i'}^t}$ of our truthful online weighted aggregation versus time slot $t$ under full feedback. Here we fix worker number $N=5$, total time slot $T=500$, and prompt number $m_t=20$.} 
    \label{fig1}
\end{figure}

Figure~\ref{fig1} is similar to Fig.~\ref{fig0} and shows how each worker $i$'s chosen probability $\frac{w_i^t}{\sum_{i'=1}^N w_{i'}^t}$ evolves over time slot $t$. As $t$ increases to 500, our mechanism manages to assign the largest probability to the most accurate worker 1 (over 0.9) and assign near-zero probabilities to the remaining ones especially for the most inaccurate worker 5, which further verifies the effectiveness of our mechanism in Definition~\ref{def3}.

\begin{figure}
    \centering \includegraphics[width=0.9\linewidth]{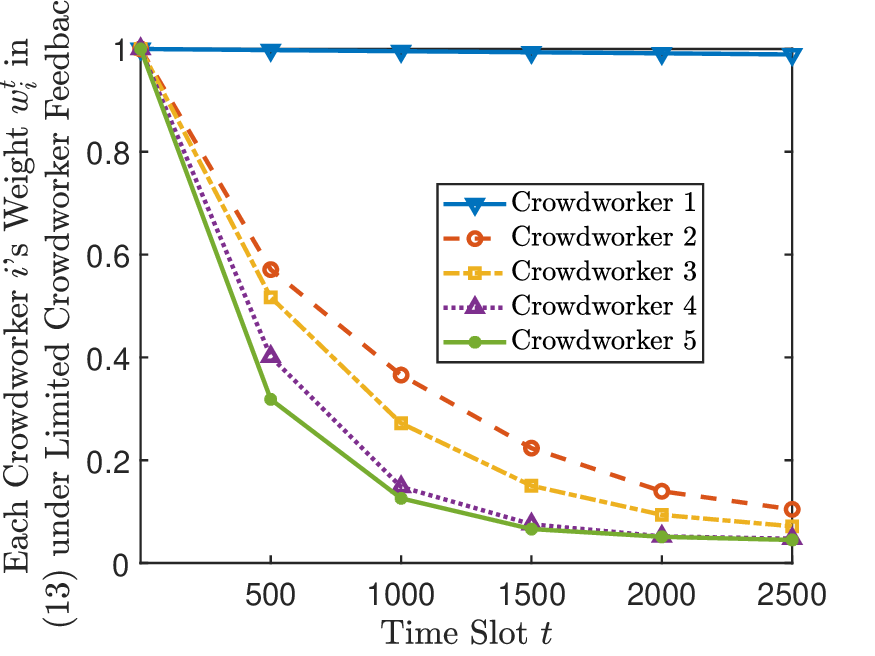}
    \caption{Each worker $i$'s weight $w_i^t$ of our truthful online mixed selection versus time slot $t$ under limited worker feedback. Here we fix worker number $N=5$, total time slot number $T=2500$, and prompt number $m_t=20$.} 
    \label{fig2}
\end{figure}

\begin{figure}
    \centering \includegraphics[width=0.9\linewidth]{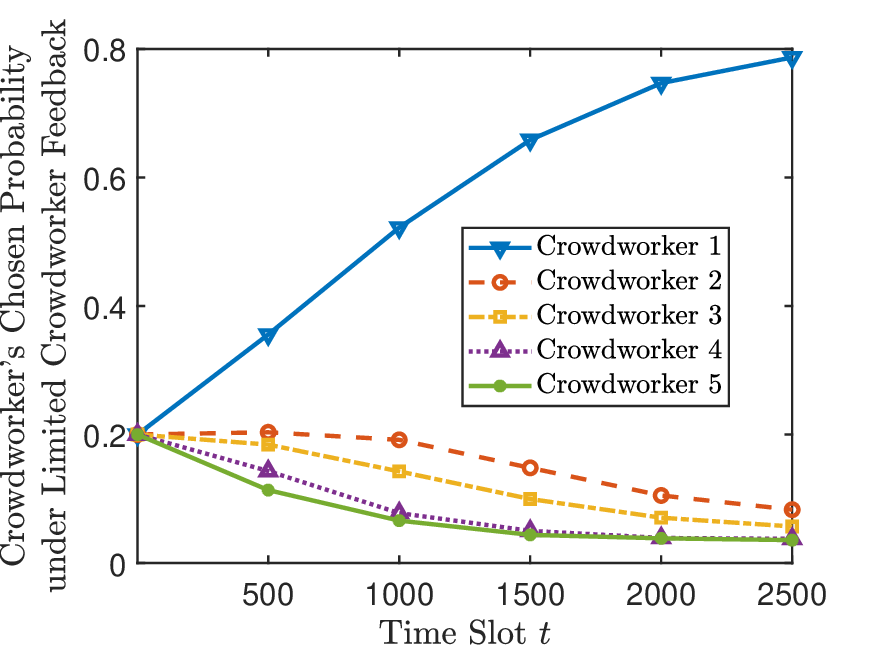}
    \caption{Each worker $i$'s chosen probability $\frac{w_i^t}{\sum_{i'=1}^N w_{i'}^t}$ of our truthful online mixed selection versus time slot $t$ under limited worker feedback. Here we fix worker number $N=5$, total time slot number $T=2500$, and prompt number $m_t=20$.} 
    \label{fig3}
\end{figure}

Figure~\ref{fig2} shows how each worker $i$'s weight $w_i^t$ evolves over time slot $t$$\in$$ [T]$ under limited worker feedback, where only one selected worker's sensing feedback is available to the platform in each time slot $t$. It indicates that our mechanism still manages to assign the largest weight to the most accurate worker 1 overtime, which is consistent with \eqref{e4-} in Definition~\ref{def3-}. Yet, as the platform only has access to one selected worker's feedback in each time slot, our mechanism needs more time slots to assign relatively small weights to the remaining ones compared to Fig.~\ref{fig0} under full feedback. 

Figure~\ref{fig3} is similar to Fig.~\ref{fig2} and shows how each worker $i$'s chosen probability $\frac{w_i^t}{\sum_{i'=1}^N w_{i'}^t}$ evolves over time slot $t$$\in$$ [T]$ under limited feedback. It indicates that our mechanism manages to assign the largest probability to the most accurate worker 1 (near 0.8), which further verifies the effectiveness of our mechanism in Definition~\ref{def3-}. Yet, under limited feedback, our mechanism needs more time slots (as $t$ increases to 2500) to converge to assign relatively small probabilities to the remaining ones.

\begin{figure}
    \centering    \includegraphics[width=0.9\linewidth]{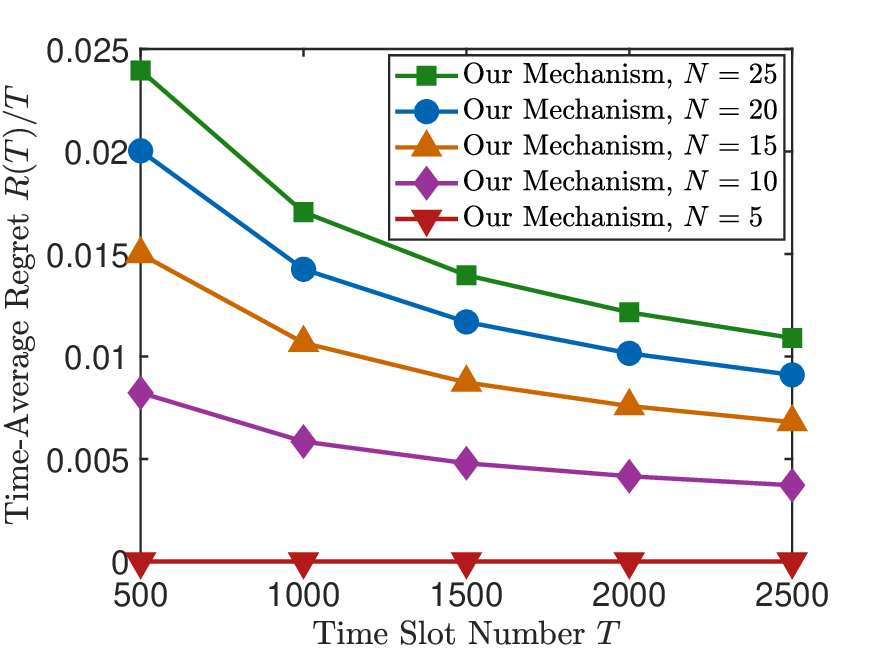}
    \caption{Time-average regrets of our mechanism under full worker feedback versus the time slot number $T$ for different worker numbers $N \in \{5, 10, 15, 20, 25\}$. Here we fix prompt number $m_t=20$ and choose the learning rate $\alpha=\frac{2}{3}\sqrt{\frac{2\ln N}{T}}$ following Theorem~\ref{Thm1}.}
    \label{fig-nfig1}
\end{figure}

{{Figure~\ref{fig-nfig1}} shows the time-average regret $R(T)/T$ of our mechanism under full worker feedback versus the time slot number $T$ for different worker numbers $N \in \{5, 10, 15, 20, 25\}$. We find that as the worker number $N$ increases, our mechanism needs more time slots to converge because the platform must explore and down-weight more inaccurate workers before identifying the most accurate one. Nevertheless, across all tested worker scales, our mechanism's time-average regret consistently decreases with $T$ and approaches 0, verifying the sublinear $\mathcal{O}(\sqrt{T})$ regret guarantee in Theorem~\ref{Thm1}. This further demonstrates that our mechanism scales gracefully with the worker number under full feedback.}

\begin{figure}
    \centering  \includegraphics[width=0.9\linewidth]{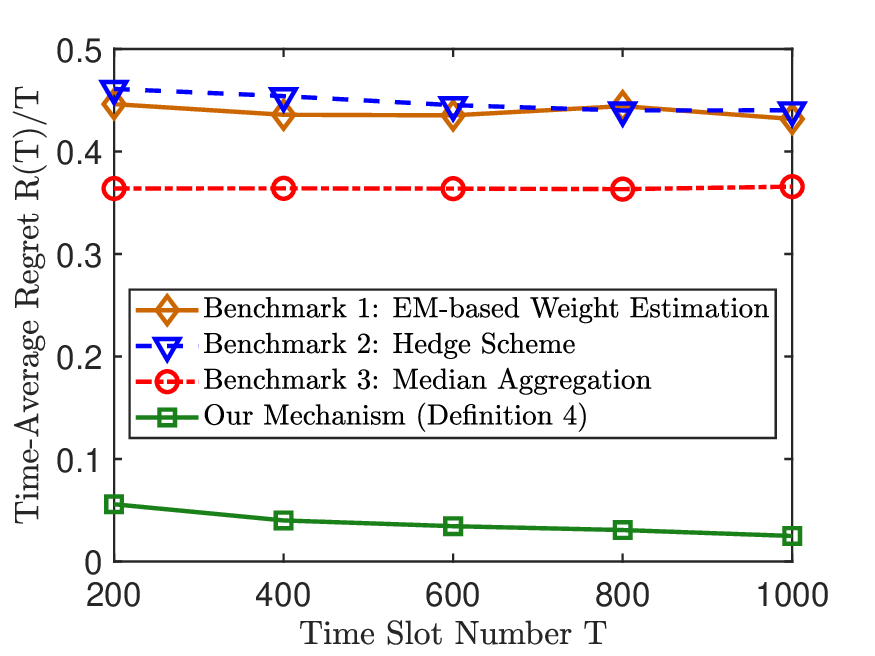}
    \caption{Time-average regrets of benchmarks 1, 2, 3 and our mechanism under full worker feedback versus the time slot number $T$, respectively. Here we choose the same learning rate $\alpha=\eta=\frac{2}{3}\sqrt{\frac{2\ln N}{T}}$ for both benchmark 2 and our method, consider a large worker scale of $N=50$ and fix prompt number $m_t=20$.} 
    \label{fig4}
\end{figure}
\begin{figure}
    \centering    \includegraphics[width=0.9\linewidth]{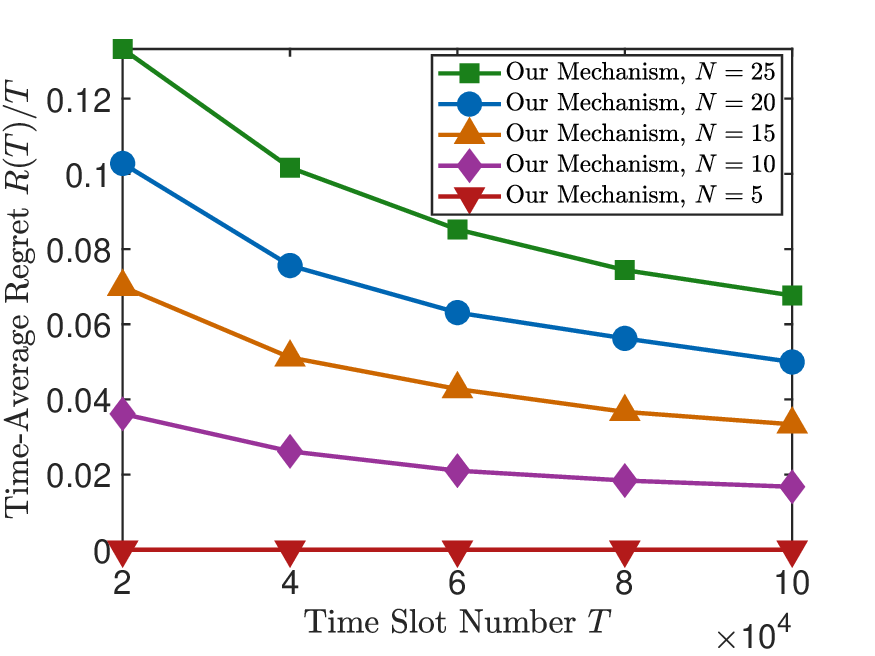}
    \caption{Time-average regrets of our mechanism under limited worker feedback versus the time slot number $T$ for different worker numbers $N \in \{5, 10, 15, 20, 25\}$. Here we fix prompt number $m_t=20$ and choose $\alpha=\sqrt{\frac{\ln N}{7NT}}$ and $\beta=2\sqrt{\frac{N\ln N}{7T}}$ following Theorem~\ref{Thm1-}.}
    \label{fig-nfig2}
\end{figure}

Figure~\ref{fig4} shows the time-average regrets $R(T)/T$ of benchmarks 1, 2, 3, and our mechanism under full worker feedback versus the time slot number $T$ under a large worker scale of $N=50$. We find that the platform's time-average regret is greatly reduced by our mechanism from the three benchmarks. Besides, time-average regrets of benchmarks 1-3 do not decrease with $T$ and are always greater than zero, respectively. Differently, our mechanism's time-average regret decreases with $T$ and tends to 0, consistent with Theorem~\ref{Thm1}.

{{Figure~\ref{fig-nfig2}} shows the time-average regret $R(T)/T$ of our mechanism under limited worker feedback versus the time slot number $T$ for different worker numbers $N \in \{5, 10, 15, 20, 25\}$. Compared with the full feedback case in Fig.~\ref{fig-nfig1}, the convergence requires substantially more time slots because only one worker's feedback is available per time slot, reducing the per-slot per-worker observation rate to $1/N$. The effect of $N$ on convergence is therefore more pronounced under limited feedback. Nevertheless, our mechanism's time-average regret still decreases with $T$ and approaches 0 across all tested worker scales, verifying the sublinear $\mathcal{O}(\sqrt{T})$ regret guarantee in Theorem~\ref{Thm1-}.}

\begin{figure}
    \centering  \includegraphics[width=0.9\linewidth]{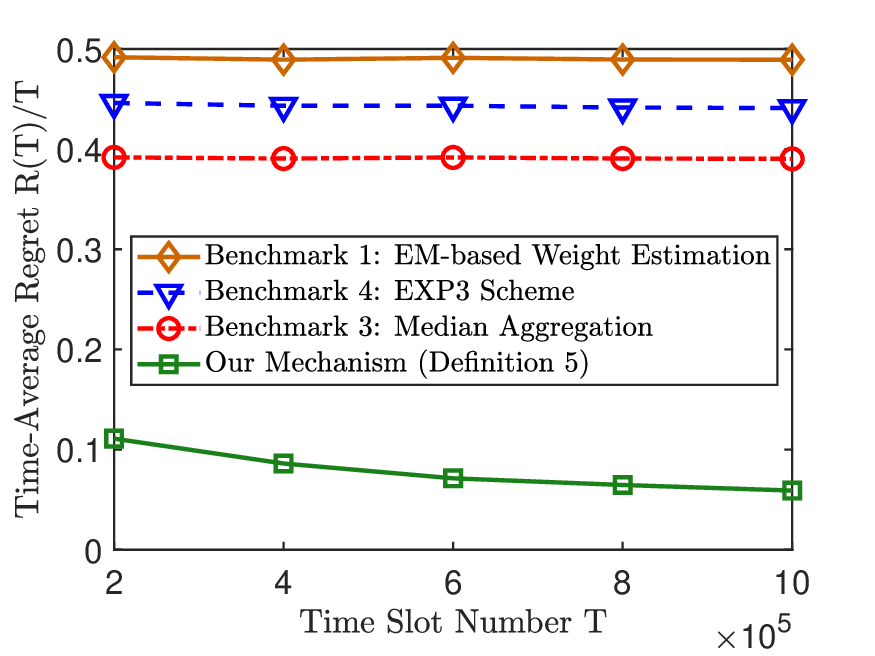}
    \caption{Time-average regrets of benchmarks 1, 3, 4 and our mechanism under limited worker feedback versus the time slot number $T$, respectively. Here we choose $\alpha = \eta =\sqrt{\frac{\ln N}{7NT}}$ and $\beta = 2\sqrt{\frac{N\ln N}{7T}}$ for both benchmark 4 and our online mixed selection scheme, consider a large worker scale of $N=50$ and fix the prompt number $m_t=20$. For a fair comparison, we randomly select 10 workers' feedback out of the total 50 for the median scheme to make a decision. } 
    \label{fig5}
\end{figure}

Figure~\ref{fig5} shows the time-average regrets $R(T)/T$ of benchmarks 1, 3, 4, and our mechanism under limited feedback versus the time slot number $T$ under a large worker scale of $N=50$. We find that the platform's time-average regret is still greatly reduced by our mechanism from the three benchmarks. Besides, time-average regrets of benchmarks 1, 3, and 4 do not decrease with $T$ and are always greater than zero, respectively. Note that compared with the full feedback case in Fig.~\ref{fig4}, our mechanism needs more time slots to obtain small enough time-average regret under limited feedback. Nevertheless, our mechanism's time-average regret still decreases with $T$ and tends to 0, consistent with Theorem~\ref{Thm1-}.

\subsection{Experiments on LLM Fine-Tuning on Real-World Datasets}\label{S6.2}

In this subsection, we further evaluate our proposed mechanisms' performance against the benchmark schemes by fine-tuning GPT-2 for a downstream cooperative spectrum sensing (CSS) task to extend from the conference version of this paper \cite{hao2025online}. We follow \cite{uvaydov2021deepsense} to use the real-world WiFi SDR dataset, which is collected using 5 USRP N210s SDRs running GNU Radio. The dataset contains about 500{,}000 transmissions over four 5~MHz non-overlapping channels and records each channel's primary user (PU) presence over time, occupying a total of 20~MHz bandwidth. It also contains heterogeneous secondary users' (SUs) probabilistic beliefs about PU presence on each subband from their local spectrum measurements as included in the original dataset, with different noise levels and biases modeling varying sensing quality. The complex baseband signal is segmented into non-overlapping windows of 32 samples and transformed with a 32-point FFT; the resulting spectrum is partitioned into four equal subbands.

\begin{figure}
    \centering
    \includegraphics[width=\linewidth]{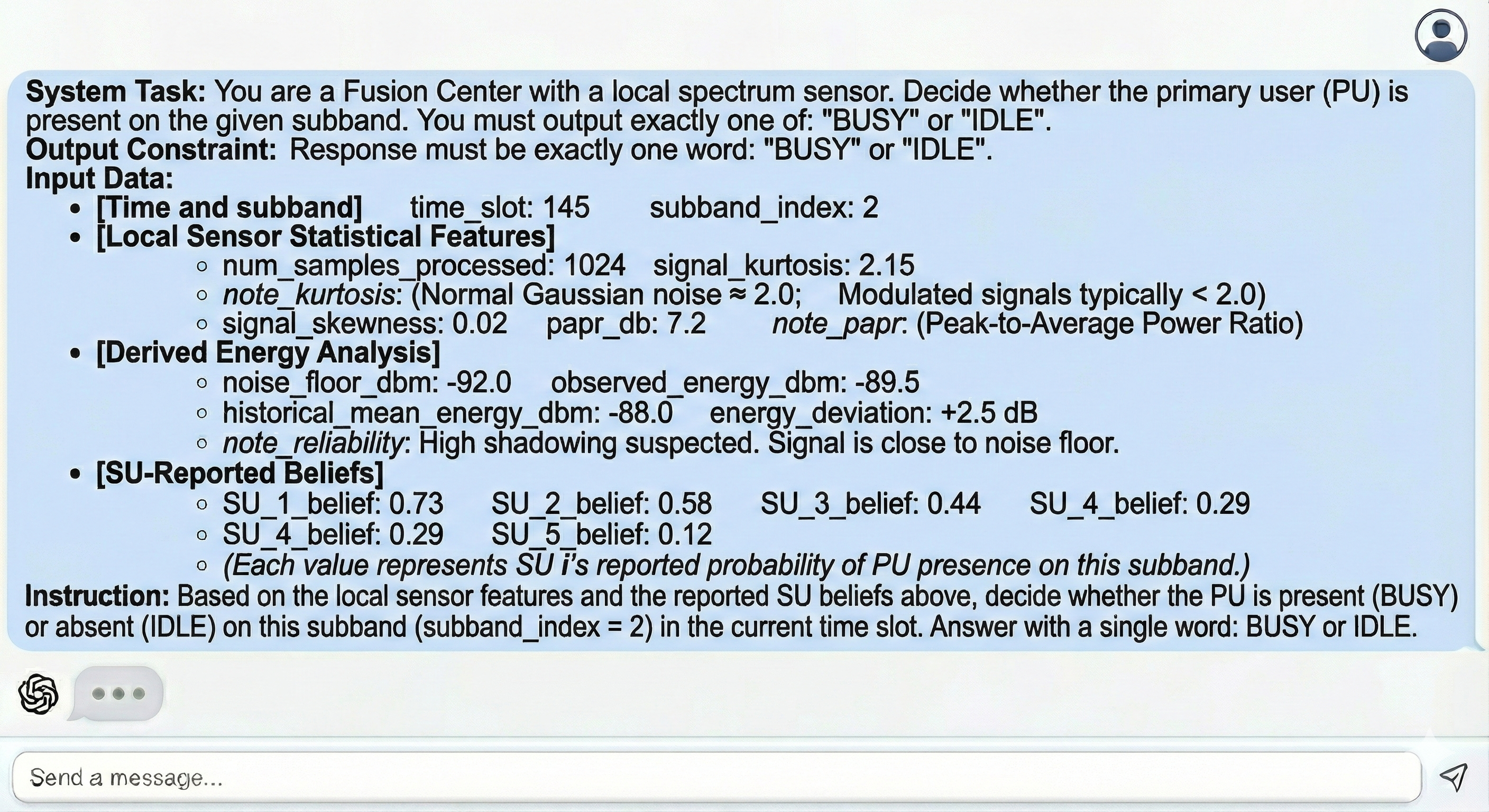}
    \caption{A prompt example for LLM fine-tuning training and testing regarding the downstream CSS task.}
    \label{fig-prmt}
\end{figure}

We choose the total time round number as $T=20$, where each round contains 20,000 data for LLM fine-tuning training and 5,000 data for testing in a time-sequential order. As shown in Fig.~\ref{fig-prmt}, we follow Section II-A to construct prompts for LLM fine-tuning training and testing. 

{We run benchmarks 1--3 and 5--6 and our truthful online aggregation mechanism in Definition~4 to aggregate SU reports, respectively, where {benchmark~5 (NeurIPS 2025 \cite{liu2025online}) is a recent state-of-the-art online weighted-majority-voting scheme also designed for LLM fine-tuning with preference data aggregated from multiple experts, and benchmark~6 is a Bayesian Beta-Binomial aggregation baseline over a latent worker-accuracy parameter.}} This yields per-subband scores interpreted as the probability of PU presence to construct human annotation sensing datasets. The sensing datasets are then used to update a CSS policy using DPO to solve a KL-regularized optimization problem against the reference policy as in \eqref{eq-dpo}, where we set $\beta=0.01$ as a parameter of evaluating the deviation from the reference policy $\pi_{\texttt{ref}}$. {GPT-2 is then fine-tuned separately using each sensing dataset with LoRA \cite{hu2022lora} adapters (rank $r=8$, scaling factor $\alpha=16$, dropout $0.05$) inserted into the attention projection matrices, while the base GPT-2 weights remain frozen. We use a maximum sequence length of 512 tokens, batch size 8, up to 3 epochs, with AdamW at learning rate $5\times 10^{-5}$. For GPT-2 (124M), the trainable LoRA parameters total approximately 0.3M, i.e., roughly 0.24\% of the full model.} At evaluation time, both the base and fine-tuned models classify by comparing the sequence log-likelihoods of the two fixed completions \texttt{BUSY} and \texttt{IDLE}.

To evaluate the fine-tuning performance under different approaches, we introduce the renowned metric win-rate in the LLM fine-tuning literature (e.g., \cite{xiong2024iterative}, \cite{rafailov2024direct}, \cite{conitzer2024social}), defined as the fraction of test prompts on which the fine-tuned policy predicts the correct label while the reference model GPT-2 fails as follows:
\begin{equation*}
\text{Win-rate}
:= \frac{1}{N_{test}}\sum_{n=1}^{N_{test}} 
\mathbf{1}\!\left(
\hat y_n^{\mathrm{pol}} = y_n \ \wedge\ 
\hat y_n^{\mathrm{ref}} \neq y_n
\right),
\end{equation*}
where $N_{test}=5000$ denotes the number of testing prompt in each round, $y_n$ denotes the ground-truth PU presence, $y_n^{\mathrm{pol}}$ denotes the LLM-generated PU presence under a policy, and $y_n^{\mathrm{ref}}$ denotes the basic GPT-2 generated PU presence.

\begin{figure}
    \centering
    \includegraphics[width=\linewidth]{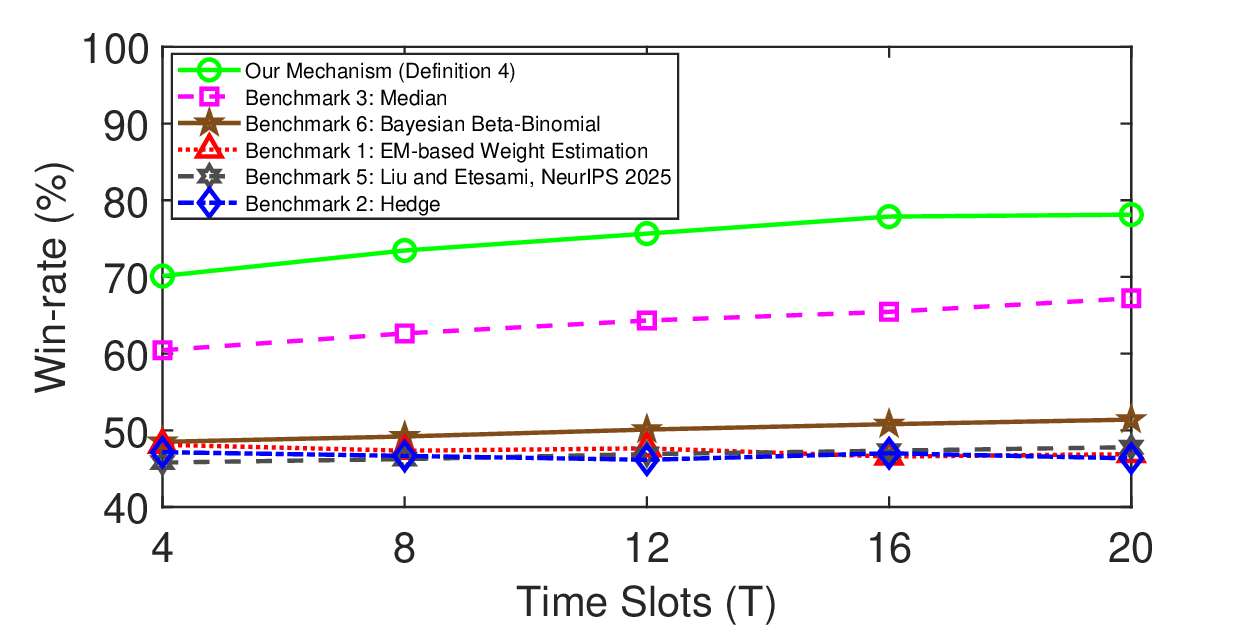}
    \caption{Win-rates after fine-tuning the GPT-2 (124M) with benchmarks 1--3, 5, 6 and our mechanism in Definition 4 versus time rounds under full SU feedback, respectively.}
    \label{fig-21}
\end{figure}
\begin{figure}
    \centering
    \includegraphics[width=\linewidth]{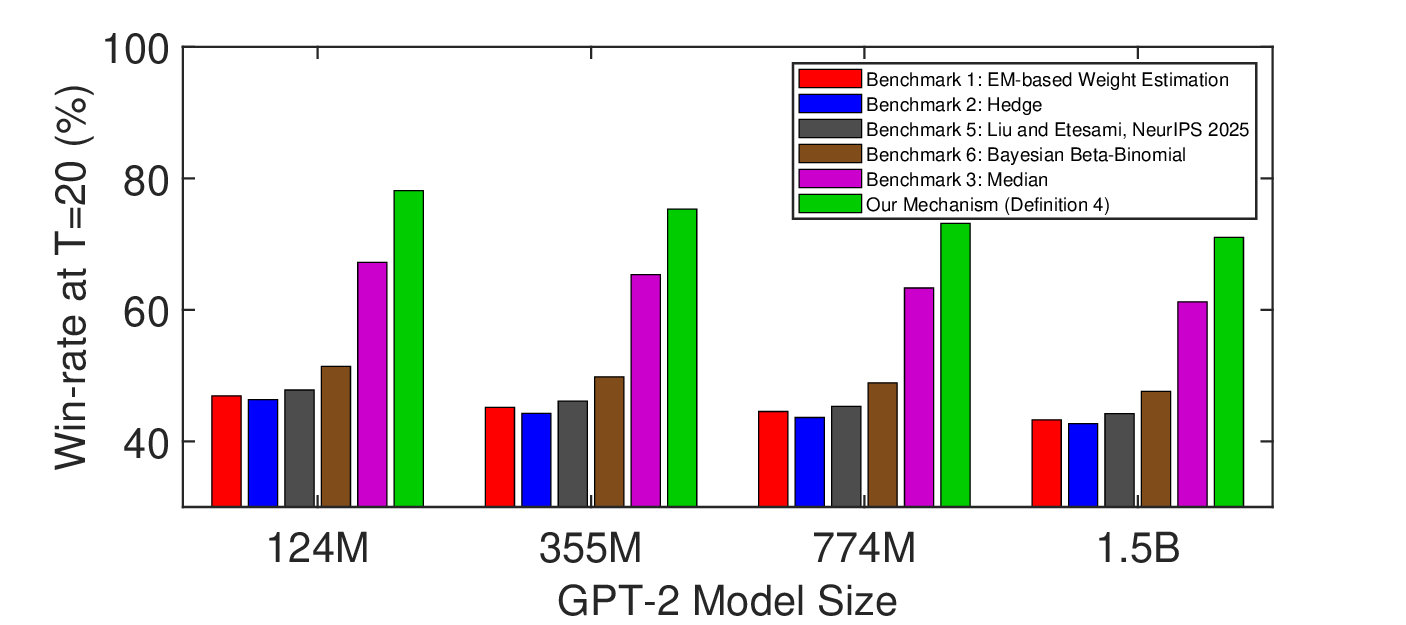}
    \caption{Win-rates after fine-tuning the GPT-2 with benchmarks 1--3, 5, 6 and our mechanism in Definition 4 at time round $T=20$ versus GPT-2 model size under full SU feedback, respectively.}
    \label{fig-22}
\end{figure}

{Figure~\ref{fig-21}} shows the win-rates versus time rounds after fine-tuning the GPT-2 (124M) with benchmarks 1--3, 5, 6 and our mechanism in Definition~4 under full SU feedback, respectively. It indicates that our mechanism achieves the highest win-rate among all the approaches. Due to untruthfulness, benchmarks 1--2 fail to obtain a meaningful human-annotation sensing dataset for fine-tuning, leading to an even degraded performance compared with the basic GPT-2 model over time. The state-of-the-art scheme (benchmark~5) and the Bayesian Beta-Binomial benchmark (benchmark~6) also underperform our mechanism, with win-rates clustering near the untruthful Hedge and EM benchmarks because neither scheme filters out strategic misreport from selfish workers. Compared to benchmark~3 of median scheme, our mechanism successfully incentivizes truthful reports from all SUs and gradually assigns the largest weights to the most accurate SUs as time evolves, leading to a substantially larger win-rate.

{Figure~\ref{fig-22}} shows win-rates after fine-tuning the GPT-2 with benchmarks 1--3, 5, 6 and our mechanism in Definition~4 at time round $T=20$ versus GPT-2 model size under full SU feedback, respectively. Note that as the GPT-2 model size increases, the basic model without fine-tuning performs better on PU presence prediction, leading to a slight decrease of win-rate for all the approaches. Nonetheless, our mechanism still achieves the highest win-rate among all the approaches across all GPT-2 model sizes, substantially outperforming both the classical and recent benchmark schemes, consistent with our truthfulness and regret guarantees in Theorems~1 and~2.

\begin{figure}
    \centering
    \includegraphics[width=\linewidth]{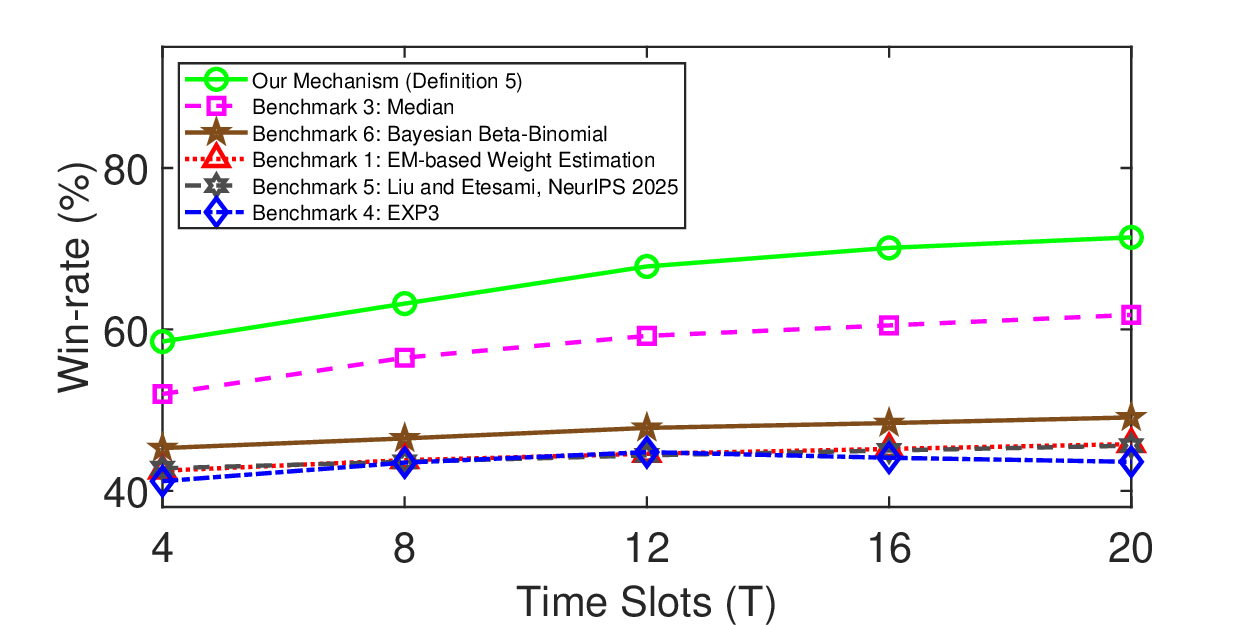}
    \caption{Win-rates after fine-tuning the GPT-2 (124M) with benchmarks 1, 3, 4, 5, 6 and our mechanism in Definition 5 versus time rounds under limited SU feedback, respectively. Here we choose limited SU number as 3 for benchmark 3 of the median scheme for a fair comparison.}
    \label{fig-23}
\end{figure}
\begin{figure}
    \centering
    \includegraphics[width=\linewidth]{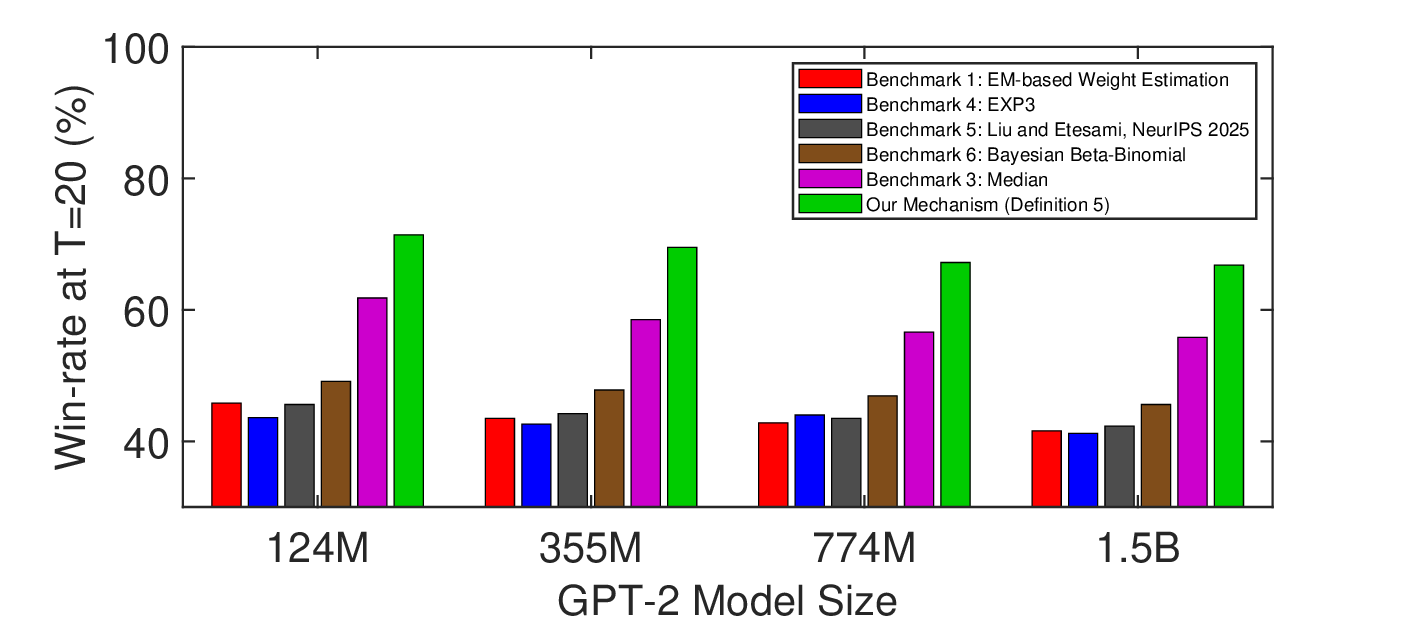}
    \caption{Win-rates after fine-tuning the GPT-2 with benchmarks 1, 3, 4, 5, 6 and our mechanism in Definition 5 at time round $T=20$ versus GPT-2 model size under limited SU feedback, respectively.}
    \label{fig-24}
\end{figure}

{Figure~\ref{fig-23}} shows the win-rates after fine-tuning the GPT-2 (124M) with benchmarks 1, 3, 4, 5, 6 and our mechanism in Definition~5 versus time rounds under limited SU feedback, respectively. With limited SU feedback, the sensing accuracy of all the approaches degrades and we expect a decrease in all the final win-rates. Nevertheless, our mechanism still achieves the highest win-rate among all the approaches, and benchmarks~5 and~6 again underperform for the same reasons of untruthfulness under strategic workers. Similar insights can also be found in Figure~\ref{fig-24}.

{Figure~\ref{fig-nfig3}} shows the win-rates of our mechanism in Definition~4 across three worker scales $N \in \{5, 15, 25\}$ and four GPT-2 model sizes, with the median benchmark included as a reference. As $N$ increases, our mechanism needs more time slots to identify and up-weight the most accurate workers, leading to a mild degradation in win-rate at fixed $T=20$ under all model sizes. Nevertheless, across all tested worker scales and model sizes, our mechanism consistently outperforms the median benchmark (best among the baselines), demonstrating that our mechanism scales gracefully with the number of workers under full feedback.

{Figure~\ref{fig-nfig4}} shows the counterpart results under limited SU feedback. Compared with the full feedback case in Fig.~\ref{fig-nfig3}, the degradation with $N$ is more pronounced because only one worker's feedback is available per time slot, reducing the per-slot per-worker observation rate to $1/N$ and slowing the convergence. Nevertheless, our mechanism still beats the median benchmark (best among the baselines) across all tested worker scales and model sizes, confirming that the mechanism remains effective under the more challenging limited-feedback setting.

\begin{figure}
    \centering
    \includegraphics[width=\linewidth]{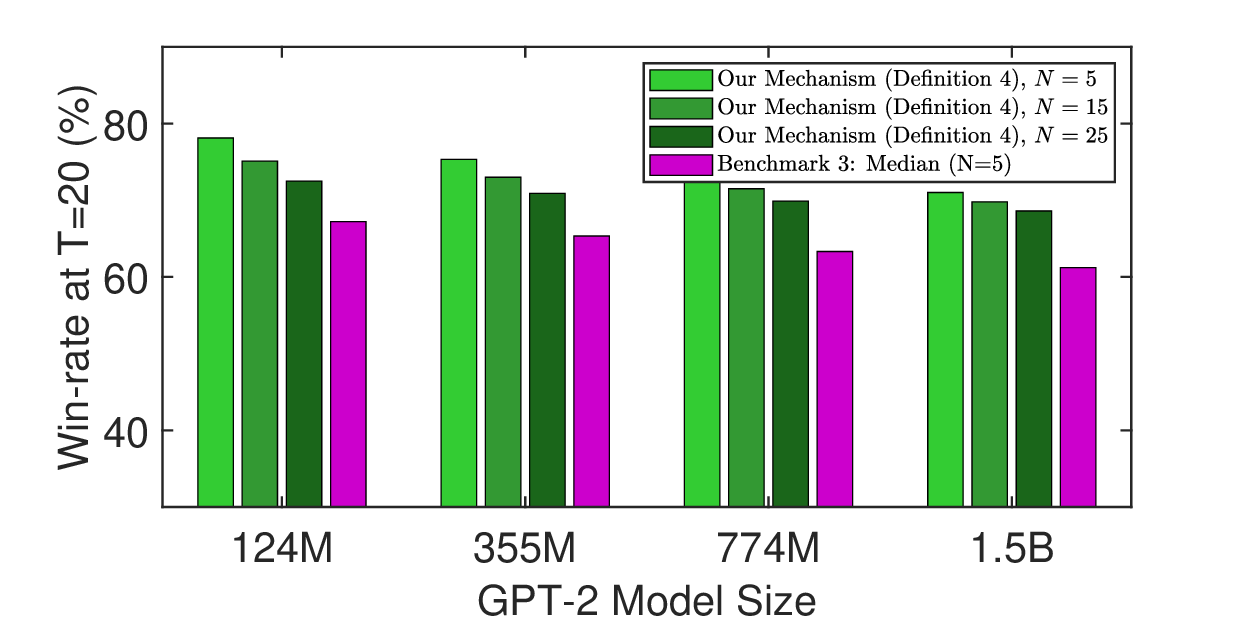}
    \caption{Win-rates after fine-tuning the GPT-2 with our mechanism in Definition 4 and benchmark 3 (median) at time round $T=20$ versus GPT-2 model size under full SU feedback, for worker numbers $N \in \{5, 15, 25\}$.}
    \label{fig-nfig3}
\end{figure}

\section{Conclusion}\label{S7}

In this paper, we study truthful online preference aggregation for LLM fine-tuning in mobile crowdsourcing against selfish workers. We prove that existing adaptive aggregation methods (e.g., EM-based weight estimation, Hedge, median, and EXP3) are not truthful and fail to identify the most accurate worker in the online setting. To address this challenge, we developed a truthful online weighted aggregation mechanism that dynamically adjusts each worker’s weight based on historical feedback accuracy. Our mechanism incentivizes truthful reporting and achieves a sublinear regret of $\mathcal{O}(\sqrt{T})$. We further extend our design to the challenging scenario where only one worker’s feedback is available per time slot, and we prove that truthful reporting and $\mathcal{O}(\sqrt{T})$ regret can still be guaranteed. Experiments on LLM fine-tuning with real-world datasets demonstrated substantial performance improvements over existing benchmark schemes.

A promising direction for future work is to consider heterogeneous coverage constraints in mobile applications. For example, a worker may only observe a subset of prompts or environmental conditions (e.g., specific routes in navigation or specific frequency bands in spectrum sensing). In such cases, the platform must assign tasks to workers under partial observability, making it non-trivial to extend our analysis to both full and limited-feedback settings.

\begin{figure}
    \centering
    \includegraphics[width=\linewidth]{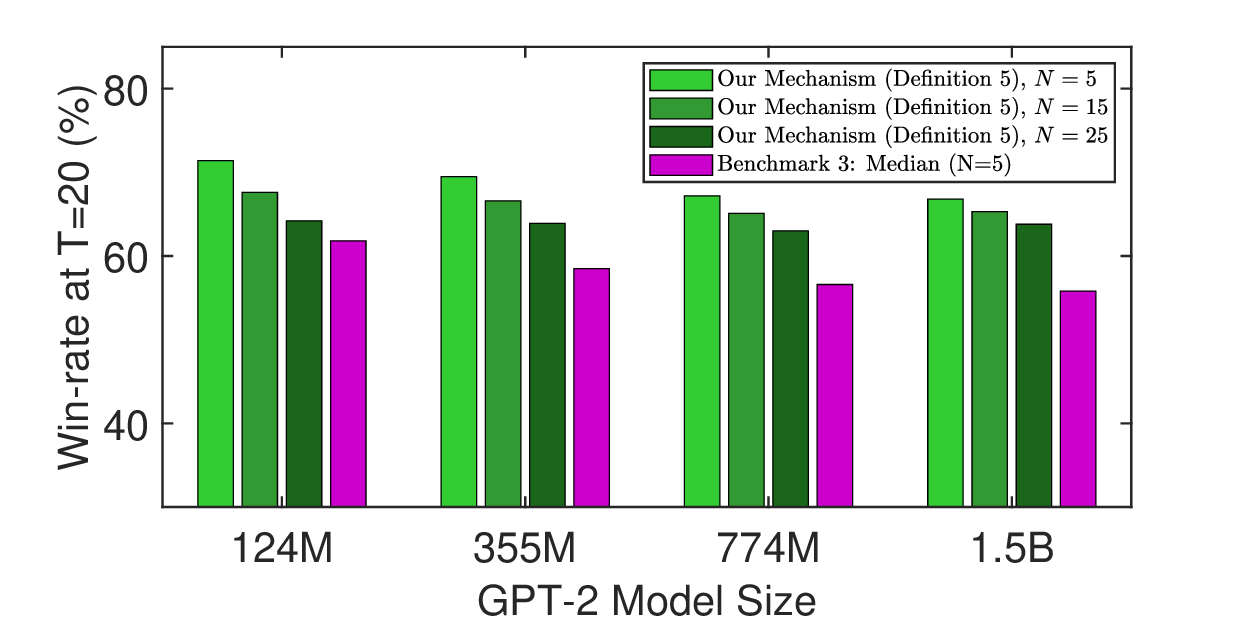}
    \caption{Win-rates after fine-tuning the GPT-2 with our mechanism in Definition 5 and benchmark 3 (median) at time round $T=20$ versus GPT-2 model size under limited SU feedback, for worker numbers $N \in \{5, 15, 25\}$.}
    \label{fig-nfig4}
\end{figure}


\begin{IEEEbiography}[{\includegraphics[width=1in,height=1.25in,clip,keepaspectratio]{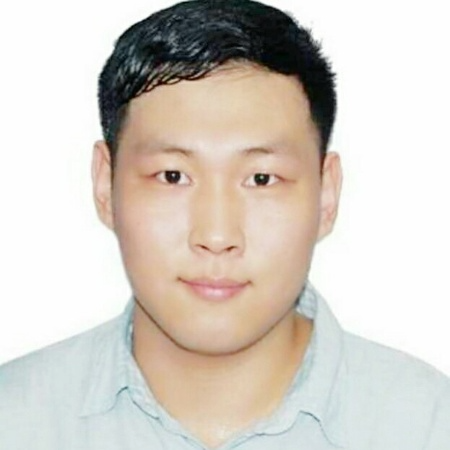}}]{Shugang Hao} (M'22)
 received the Ph.D. degree from Singapore University of Technology and Design (SUTD) in 2022. He is a postdoctoral research fellow at SUTD from Sep. 2022.  
His research interests are LLM for networks, networking for LLM, LLM fine-tuning, game theory and mechanism design. He served as the web chair of ACM SenSys 2024, the local arrangement chair of AIoTSys 2024 and the local arrangement chair of IEEE WiOpt 2023. 
\end{IEEEbiography}

\begin{IEEEbiography}[{\includegraphics[width=1in,height=1.25in,clip,keepaspectratio]{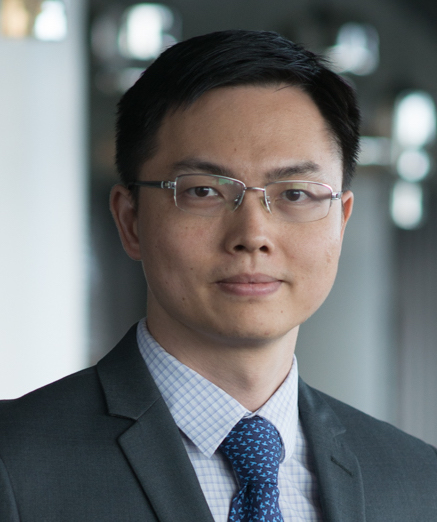}}]{Lingjie Duan}
(S'09-M'12-SM'17) received the
Ph.D. degree from The Chinese University of
Hong Kong in 2012. He is a full Professor in the Internet of Things Thrust and the Artificial Intelligence Thrust at Hong Kong University of Science and Technology, Guangzhou. He was an Associate Professor and Associate Head of Pillar
of Engineering Systems and Design with the
Singapore University of Technology and Design
(SUTD). In 2011, he was a Visiting Scholar at University
of California at Berkeley, Berkeley, CA,
USA. His research interests include network economics
and game theory, cognitive communications,
and cooperative networking. He is an Associate
Editor of IEEE Transactions on Mobile Computing and IEEE Transactions on Networking. He was an
Editor of IEEE Transactions on Wireless Communications and IEEE Communications Surveys and Tutorials. He also served
as a Guest Editor of the IEEE Journal on Selected Areas in Communications Special Issue on Human-in-the-Loop Mobile Networks, as well
as IEEE Wireless Communications Magazine.  He received the SUTD Excellence
in Research Award in 2016 and the 10th IEEE ComSoc Asia-Pacific
Outstanding Young Researcher Award in 2015. He served as the general chair of IEEE WiOpt 2023.
\end{IEEEbiography}

\newpage
 ~\newpage

\appendices

\section{Proof of Lemma~\ref{L0}}\label{A00}

First, we prove that Benchmark~1 (EM-based Weight Estimation) is not truthful.
We construct a counterexample with horizon $T=2$ showing that truthful reporting is not a dominant strategy. The total utility for worker $k$ is
\[
u_k=\sum_{t=1}^2 w_k^t = w_k^1+w_k^2.
\]
Since the initial weights are uniform ($w_i^1=w_0,\ \forall i$) and independent of reports, maximizing $u_k$ is equivalent to maximizing the expected weight at the next time step, i.e., $\mathbb{E}[w_k^2]$.

Consider time $t=1$ and a worker $k$ with private belief $q:=\Pr(p^1=1)\in(0,1)$ about the hidden binary outcome $p^1\in\{0,1\}$. We compare two strategies:
\begin{itemize}
    \item \textit{Strategy $S_{\text{truth}}$:} report truthfully by sampling from the belief, i.e., $\hat y_k^1\sim \mathrm{Bernoulli}(q)$ (so $\Pr(\hat y_k^1=1)=q$).
    \item \textit{Strategy $S_{\text{lie}}$:} report a constant extreme to align with the majority, $\hat y_k^1=1$ (\emph{Always High}).
\end{itemize}

We instantiate the EM benchmark with a standard Dawid--Skene likelihood. Conditional on the latent truth $p^1$, each worker $i$ reports $\hat y_i^1\in\{0,1\}$ correctly with probability $w_i\in(0,1)$:
\[
\Pr(\hat y_i^1=p^1\mid w_i)=w_i,
\qquad
\Pr(\hat y_i^1\neq p^1\mid w_i)=1-w_i.
\]
We place an independent Beta prior $w_i\sim \mathrm{Beta}(\alpha,\beta)$ with fixed $\alpha,\beta>0$. The prior on the truth is $\Pr(p^1=1)=\pi\in(0,1)$; for simplicity take $\pi=\frac12$.
Assume all initial weights are identical and satisfy the standard reliability assumption:
\[
w_i^{(1)}=w_i^1=w_0\in\Big(\frac12, 1\Big)\qquad \forall i.
\]

Given current parameter estimates $\mathbf w^{(1)}$, the E-step computes the posterior
{\footnotesize
\begin{align*}
&\gamma := \Pr(p^1=1\mid \hat{\mathbf y}^1,\mathbf w^{(1)}) \\
&= 
\frac{\pi\prod_{i=1}^N \Pr(\hat y_i^1\mid p^1=1,w_0)}
{\pi\prod_{i=1}^N \Pr(\hat y_i^1\mid p^1=1,w_0)+(1-\pi)\prod_{i=1}^N \Pr(\hat y_i^1\mid p^1=0,w_0)},
\end{align*}
}
where $\hat{\mathbf y}^1=(\hat y_1^1,\dots,\hat y_N^1)$.
The M-step updates each worker's reliability by the posterior-mean rule
\[
w_i^{2}
=
\frac{\alpha + \mathbb E[\mathbf 1\{\hat y_i^1=p^1\}\mid \hat{\mathbf y}^1,\mathbf w^{(1)}]}
{\alpha+\beta+1},
\]
where
\[
\mathbb E[\mathbf 1\{\hat y_i^1=p^1\}\mid \hat{\mathbf y}^1,\mathbf w^{(1)}]
=
\gamma\,\mathbf 1\{\hat y_i^1=1\} + (1-\gamma)\,\mathbf 1\{\hat y_i^1=0\}.
\]

Assume the other $N-1$ workers all report $\hat y_i^1=1$ at time $t=1$. Define
\begin{align*}
&\gamma_1 := \Pr(p^1=1\mid \hat y_k^1=1,\ \hat y_{-k}^1=\mathbf 1,\ \mathbf w^{(1)}),  \\
&\gamma_0 := \Pr(p^1=1\mid \hat y_k^1=0,\ \hat y_{-k}^1=\mathbf 1,\ \mathbf w^{(1)}).
\end{align*}
Because all workers share the same $w_0$, the likelihood depends only on the number of ones.
Let $m$ be the number of ones among the $N$ reports at $t=1$. Under Dawid--Skene,
\[
\frac{\Pr(\hat{\mathbf y}^1\mid p^1=1,w_0)}{\Pr(\hat{\mathbf y}^1\mid p^1=0,w_0)}
=
\left(\frac{w_0}{1-w_0}\right)^{\,2m-N}.
\]
With $\pi=\frac12$, Bayes' rule gives
\[
\gamma(m)=\Pr(p^1=1\mid m)
=
\frac{1}{1+\left(\frac{1-w_0}{w_0}\right)^{\,2m-N}}.
\]
In our environment,
\[
m=
\begin{cases}
N, & \text{if } \hat y_k^1=1,\\
N-1, & \text{if } \hat y_k^1=0,
\end{cases}
\]
so $\gamma_1=\gamma(N)$ and $\gamma_0=\gamma(N-1)$. Since $w_0>\frac12$ implies $\frac{1-w_0}{w_0}<1$, we have
\[
\gamma_1 > \gamma_0 > \frac12,
\qquad
\gamma_0+\gamma_1 > 1.
\]

Now compare the expected updated weight for worker $k$.

1) If $k$ lies (i.e., $\hat y_k^1=1$), then $\mathbf 1\{\hat y_k^1=1\}=1$, so
\[
\mathbb E[\mathbf 1\{\hat y_k^1=p^1\}\mid \hat{\mathbf y}^1,\mathbf w^{(1)}]=\gamma_1,
\]
and hence
\[
w_{k,\text{lie}}^2=\frac{\alpha+\gamma_1}{\alpha+\beta+1}.
\]

2) If $k$ is truthful (i.e., $\hat y_k^1\sim\mathrm{Bernoulli}(q)$), then conditioning on $\hat y_k^1$:
\[
w_{k,\text{truth}}^2=
\begin{cases}
\frac{\alpha+\gamma_1}{\alpha+\beta+1}, & \text{if } \hat y_k^1=1,\\[6pt]
\frac{\alpha+(1-\gamma_0)}{\alpha+\beta+1}, & \text{if } \hat y_k^1=0,
\end{cases}
\]
so taking expectation over $\hat y_k^1$ gives
\[
\mathbb E[w_{k,\text{truth}}^2]
=
q\cdot \frac{\alpha+\gamma_1}{\alpha+\beta+1}
+
(1-q)\cdot \frac{\alpha+1-\gamma_0}{\alpha+\beta+1}.
\]
Therefore,
\begin{align*}
&\mathbb E[w_{k,\text{lie}}^2]-\mathbb E[w_{k,\text{truth}}^2] \\
=&
(1-q)\cdot \frac{\gamma_1+\gamma_0-1}{\alpha+\beta+1}.
\end{align*}
Since $\gamma_0+\gamma_1>1$, the numerator is positive. For any belief $q<1$, the difference is strictly positive. It follows that
\[
\mathbb E[w_{k,\text{lie}}^2]>\mathbb E[w_{k,\text{truth}}^2].
\]
Thus $u_k(\text{lie})>u_k(\text{truth})$, so truthful reporting is not a dominant strategy. This proves that Benchmark~1 is not truthful.

Next, we prove that Benchmark~1 leads to a non-vanishing regret.
Fix any horizon $T$ and consider the following construction. There exists one worker $o\in[N]$ such that
\[
\mathcal{P}_o(y_{l_j}^t > y_{l_j'}^t \mid x_j^t) = p_j^t \quad \forall j\in[m_t], t\in[T],
\]
i.e., worker $o$ is always correct. For every other worker $i\ne o$, assume they always report the opposite signal:
\[
\hat{\mathcal{P}}_i(y_{l_j}^t > y_{l_j'}^t \mid x_j^t) = 1-p_j^t \quad \forall i\ne o, j\in[m_t], t\in[T].
\]
Assuming $N \ge 3$, the $N-1$ adversarial workers form a majority. Under the EM instantiations with $w_0 > 0.5$, due to the existence of multiple stationary points of the likelihood function and the dependence of EM on initialization and sample realizations, there exists a realization (and corresponding EM trajectory) under which the algorithm reinforces the majority cluster because they are statistically consistent with each other. Consequently, the inferred latent variable converges to the majority opinion:
\[
\hat{\mathcal{P}}(y_{l_j}^t > y_{l_j'}^t \mid x_j^t) = 1-p_j^t.
\]
Accordingly, the best fixed worker in hindsight is $i^*=o$, which yields
\[
\min_{i\in[N]} \sum_{t=1}^T \frac{1}{m_t}\sum_{j=1}^{m_t} (\mathcal{P}_i - p_j^t)^2 = 0.
\]
However, the cumulative aggregation loss over $T$ slots is
\[
\sum_{t=1}^{T} \frac{1}{m_t}\sum_{j=1}^{m_t} (\hat{\mathcal{P}} - p_j^t)^2 
= T.
\]
Finally, the regret of the EM-based scheme is
\begin{align*}
R_1(T) &= T - 0 = \mathcal{O}(T),
\end{align*}
which implies $\lim_{T\to\infty} \frac{R_1(T)}{T} > 0$. We finish the proof.

\section{Proof of Lemma~\ref{L1}}\label{A0}

First, we prove that the Hedge Scheme is not truthful.
We want to show that for any worker number $N\ge 2$ and time horizon $T\ge 2$, there exists a belief $q\in(0,1)\setminus\{\tfrac12\}$ and a strategy profile of other workers such that worker $i$ achieves a strictly higher expected long-term utility by misreporting in the first time slot. Assume $m_t=1$ for all $t\in[T]$.

Let worker $i$ have private belief $q=\Pr(p^1=1)$ with $q\notin\{0,\tfrac12,1\}$. The worker's long-term utility is
\[
u_i=\sum_{t=1}^T w_i^t.
\]
The Hedge weight update rule is
\[
w_i^{t+1}=w_i^t\cdot \exp\!\left(-\eta\big(\hat P_i^t-p^t\big)^2\right),\qquad w_i^1=1,
\]
where $\eta>0$ is the learning rate.

Due to the multiplicative structure, for any $t\ge 2$,
\[
w_i^t=w_i^2\cdot \prod_{\tau=2}^{t-1}\exp\!\left(-\eta\big(\hat P_i^\tau-p^\tau\big)^2\right).
\]
Define the future weight multiplier
\[
M_i:=1+\sum_{t=2}^T \prod_{\tau=2}^{t-1}\exp\!\left(-\eta\big(\hat P_i^\tau-p^\tau\big)^2\right),
\]
which satisfies $M_i\ge 1$. Then the utility can be written as
\[
u_i = 1 + w_i^2\cdot M_i.
\]
Consider two reporting strategies for worker $i$ in slot $t=1$:
\begin{itemize}
\item Truthful strategy $S_T$: report $\hat P_i^1=q$.
\item Deviating strategy $S_D$: report $\hat P_i^1=r\neq q$, where $r$ is chosen to maximize $\mathbb E[w_i^2]$.
\end{itemize}
We then have
\[
\mathbb E[u_i\mid S]=1+\mathbb E[w_i^2\mid S]\cdot \mathbb E[M_i],\qquad S\in\{S_T,S_D\}.
\]

We analyze $\mathbb E[w_i^2]$ as a function of the first report $r$. Since $p^1\in\{0,1\}$ with $\Pr(p^1=1)=q$, we have
\[
\mathbb E[w_i^2]= q\exp\!\left(-\eta(r-1)^2\right) + (1-q)\exp\!\left(-\eta r^2\right) =: F(r).
\]
Its derivative is
\[
F'(r)=-2\eta\left[q(r-1)e^{-\eta(r-1)^2}+(1-q)r e^{-\eta r^2}\right].
\]
Evaluating at $r=q$ yields
\[
F'(q)=-2\eta\,q(1-q)\left(e^{-\eta q^2}-e^{-\eta(1-q)^2}\right).
\]
If $q>\tfrac12$, then $q^2>(1-q)^2$ and thus $e^{-\eta q^2}<e^{-\eta(1-q)^2}$, implying $F'(q)>0$.
If $q<\tfrac12$, then $q^2<(1-q)^2$ and thus $e^{-\eta q^2}>e^{-\eta(1-q)^2}$, implying $F'(q)<0$.
Therefore, for $q\neq \tfrac12$, the truthful report $r=q$ is not a local maximizer of $F(r)$; hence there exists $r^\star\neq q$ such that
\[
\mathbb E[w_i^2\mid S_D] > \mathbb E[w_i^2\mid S_T].
\]
Since $\mathbb E[M_i]>0$, it follows that $\mathbb E[u_i\mid S_D]>\mathbb E[u_i\mid S_T]$, proving that the Hedge scheme is not truthful.

\section{Proof of Lemma~\ref{L1-1}}\label{A4}

We want to prove $R_3(T) = \mathcal{O}(T)$ with a possible sequence of workers' preferences. In particular, we consider $\mathcal{P}_o(y_{l_j}^t \succ y_{l_j'}^t|x_j^t) = p_j^t$ holds for one particular $o \in [N]$ with any $j \in [m_t]$ and $t \in [T]$. Further, we consider  $(\mathcal{\hat{P}}_{j, k_m}^t-p_j^t)^2=c_j^t$ for $j \in [m_t]$ and $t \in [T]$, where $\mathcal{\hat{P}}_{j, k_m}^t$ denotes the median of workers' feedback $\{\mathcal{\hat{P}}_i(y_{l_j}^t \succ y_{l_j'}^t|x_j^t)\}_{i=1}^N$ and $c_j^t \in [\frac{1}{2}, 1]$. Accordingly, we have the best-fixed worker in hindsight is $i^* = o$, which brings
    \begin{align*}
        \min_{i \in [N]} \sum_{t=1}^T\frac{1}{m_t}\sum_{j=1}^{m_t}\bigg(\mathcal{P}_i(y_{l_j}^t \succ y_{l_j'}^t|x_j^t) - p_j^t\bigg)^2 = 0.
    \end{align*}
    However, with the platform's median scheme, we have the cumulative aggregation loss over $T$ slots as follows:
    \begin{align*}    
    &\sum_{t=1}^T\frac{1}{m_t}\sum_{j=1}^{m_t}\bigg(\sum_{i=1}^N \frac{w_i^{t}\mathcal{\hat{P}}_i(y_{l_j}^t \succ y_{l_j'}^t|x_j^t)}{\sum_{i'=1}^N w_{i'}^{t}} - p_j^t\bigg)^2 \\
    =&\sum_{t=1}^T\frac{1}{m_t}\sum_{j=1}^{m_t}\bigg(\mathcal{\hat{P}}_{j, k_m}^t - y_j^{t}\bigg)^2 \\
=&  \sum_{t=1}^T \frac{1}{m_t}\sum_{j=1}^{m_t} c_j^t = \mathcal{O}(T),
    \end{align*}
    where the last equality holds because each $c_j^t \in [\frac{1}{2}, 1]$ and $\sum_{t=1}^T \frac{1}{m_t}\sum_{j=1}^{m_t} c_j^t$ does not vanish as $T \to \infty$.
    Finally, we have the regret of the median scheme as follows:
\begin{align*}
    R_3(T) =& \sum_{t=1}^T\frac{1}{m_t}\sum_{j=1}^{m_t}\bigg(\sum_{i=1}^N  \frac{w_i^{t}\mathcal{\hat{P}}_i(y_{l_j}^t \succ y_{l_j'}^t|x_j^t)}{\sum_{i'=1}^N w_{i'}^{t}} - p_j^t\bigg)^2 \\ 
    &- \min_{i \in [N]} \sum_{t=1}^T\frac{1}{m_t}\sum_{j=1}^{m_t}\bigg(\mathcal{P}_i(y_{l_j}^t \succ y_{l_j'}^t|x_j^t) - p_j^t\bigg)^2 \\
    =& \mathcal{O}(T).
\end{align*}
We then finish the proof.

\section{Proof of Proposition~\ref{L2}}\label{A5}

Since $g_i(\cdot)$ is strictly increasing, maximizing the utility
$u_i = g_i\!\left(\mathbb{E}\!\left[\sum_{t=1}^{T} w_i^t\right]\right)$
is equivalent to maximizing the expected cumulative weight
$\mathbb{E}\!\left[\sum_{t=1}^{T} w_i^t\right]$.
 According to our system model in Section~\ref{S2}, each worker believes that $p_j^t$$\sim$$\texttt{Bernoulli}(\mathcal{P}_i(y_{l_j}^t$$\succ$$y_{l_j'}^t|x_j^t))$, we have expectation on $ w_{i}^{t+1}$ in \eqref{e4} over $p_j^t$ is
    \begin{align*}
        &\mathbb{E}[w_{i}^{t+1}] \\
        =&  w_{i}^{t}\frac{1}{m_t}\sum_{j=1}^{m_t}\bigg[ 1  -\alpha \mathcal{P}_i(y_{l_j}^t \succ y_{l_j'}^t|x_j^t)(\mathcal{\hat{P}}_i(y_{l_j}^t \succ y_{l_j'}^t|x_j^t)- 1)^2 \\
        &- \alpha (1-\mathcal{P}_i(y_{l_j}^t \succ y_{l_j'}^t|x_j^t))(\mathcal{\hat{P}}_i(y_{l_j}^t \succ y_{l_j'}^t|x_j^t)- 0)^2 \bigg] \\
        =& w_{i}^{t}\frac{1}{m_t}\sum_{j=1}^{m_t}\bigg[1 - \alpha (\mathcal{\hat{P}}_i(y_{l_j}^t \succ y_{l_j'}^t|x_j^t) - \mathcal{P}_i(y_{l_j}^t \succ y_{l_j'}^t|x_j^t))^2 \\
        &- \alpha (\mathcal{P}_i(y_{l_j}^t \succ y_{l_j'}^t|x_j^t) - \mathcal{P}_i^2(y_{l_j}^t \succ y_{l_j'}^t|x_j^t)) \bigg],
    \end{align*}
    which is maximized at $\mathcal{\hat{P}}_i^*(y_{l_j}^t \succ y_{l_j'}^t|x_j^t) = \mathcal{P}_i(y_{l_j}^t \succ y_{l_j'}^t|x_j^t)$. To obtain the largest possible accumulative weight, each worker will truthfully feedback his preference in the first time slot and all the following time slots because any deviation will lead to smaller weights of the next and all the following time slots. We finish the proof.

\section{Proof of Lemma~\ref{Benchmark 4}} \label{A5+}

First, we prove that the EXP3 scheme is not truthful. We aim to show that for any number of workers $N \geq 2$, there exists a time horizon $T \geq 2$, a belief $q \in (0,1) \setminus \{\frac{1}{2}\}$ and a strategy profile of other workers such that worker $i$ achieves a strictly higher expected long-term utility by misreporting in the first time slot.

Consider the case where $m_t = 1$ for all $t \in [T]$. Let worker $i$ have a private belief $q = \mathcal{P}_i^1 = \Pr(p^1 = 1)$, with $q \notin \{0, \frac{1}{2}, 1\}$. The worker's long-term utility is defined as the sum of its weights:
\[
u_i = \sum_{t=1}^{T} w_i^t.
\]

The EXP3 weight update rule for a selected worker $I_t$ is:
\[
w_i^{t+1} = \begin{cases}
w_i^t \cdot \exp\left(-\eta \cdot \tilde{\ell}_i^t\right), & \text{if } i = I_t, \\
w_i^t, & \text{otherwise},
\end{cases}
\]
with the unbiased loss estimator $\tilde{\ell}_i^t$ given by:
\[
\tilde{\ell}_i^t = \begin{cases}
\frac{\hat{\ell}_i^t}{(1-\beta)\theta_i^t + \beta/N}, & \text{if } i = I_t, \\
0, & \text{otherwise},
\end{cases}
\]
where $\hat{\ell}_i^t = (\hat{\mathcal{P}}_i^t - p^t)^2$ and $\theta_i^t = \frac{w_i^t}{\sum_{i'=1}^{N} w_{i'}^t}$.

Due to the multiplicative structure of the update, the weight at any time $t \geq 2$ depends on the sequence of selections and reports. The utility can be expressed as:
\[
u_i = 1 + \sum_{t=2}^{T} w_i^t.
\]

Now, consider two reporting strategies for worker $i$ in the first time slot $t=1$:
\begin{itemize}
    \item \emph{Truthful strategy $S_T$}: Report $\hat{\mathcal{P}}_i^1 = q$.
    \item \emph{Deviating strategy $S_D$}: Report $\hat{\mathcal{P}}_i^1 = r^* \neq q$, where $r^*$ is chosen to maximize the expected weight in the subsequent slots.
\end{itemize}

To prove that the EXP3 scheme is not truthful, it is sufficient to exhibit a counter-example for some finite horizon $T \geq 2$. In particular, it is enough to take $T = 2$. In this case, the worker's long-term utility simplifies to
\[
u_i = w_i^1 + w_i^2 = 1 + w_i^2,
\]
since $w_i^1 = 1$ for all $i$. Therefore, for $T=2$, maximizing the expected long-term utility $\mathbb{E}[u_i|S]$ is equivalent to maximizing the expected weight after the first slot, $\mathbb{E}[w_i^2 \mid S]$, for $S \in \{S_T, S_D\}$.

To show that deviation is profitable, we analyze $\mathbb{E}[w_i^2 \mid S]$. Note that $w_i^2$ only changes if worker $i$ is selected in $t=1$ (i.e., $I_1 = i$). Let $\theta_i^1 = \frac{1}{N}$ be the initial selection probability.

If $I_1 = i$, then:
\begin{align*}
  w_i^2 = w_i^1 \cdot \exp\left(-\eta \cdot \tilde{\ell}_i^1\right)
  &= \exp\left(-\eta \cdot \frac{\hat{\ell}_i^1}{(1-\beta)\theta_i^1 + \beta/N}\right) \\
  &= \exp\left(-\eta \cdot \frac{\hat{\ell}_i^1}{(1-\beta)\cdot \frac{1}{N} + \frac{\beta}{N}}\right) \\
  &= \exp\left(-\eta N \cdot \hat{\ell}_i^1\right).
\end{align*}
If $I_1 \neq i$, then $w_i^2 = w_i^1 = 1$.

Therefore, the expected weight after the first slot is:
\[
\mathbb{E}[w_i^2 \mid S] = \theta_i^1 \cdot \mathbb{E}\!\left[\exp\left(-\eta N \cdot \hat{\ell}_i^1\right) \mid S\right] + (1 - \theta_i^1) \cdot 1,
\]
i.e.,
\[
\mathbb{E}[w_i^2 \mid S] = \frac{1}{N} \cdot \mathbb{E}\!\left[\exp\left(-\eta N \cdot \hat{\ell}_i^1\right) \mid S\right] + \frac{N-1}{N}.
\]

Define
\[
F(r) = \mathbb{E}\!\left[\exp\left(-\eta N \cdot \hat{\ell}_i^1\right) \,\middle|\, \hat{\mathcal{P}}_i^1 = r\right].
\]
We have:
\[
F(r) = q \cdot \exp\left(-\eta N (r-1)^2\right) + (1-q) \cdot \exp\left(-\eta N r^2\right).
\]

Its derivative is:
\begin{align*}
    &F'(r) = -2\eta N \\
    &  \left[ q(r-1)\exp\left(-\eta N (r-1)^2\right) + (1-q)r\exp\left(-\eta N r^2\right) \right].
\end{align*}

Evaluating at the truthful report $r = q$:
\[
F'(q) = -2\eta N \cdot q(1-q) \left[ \exp\left(-\eta N q^2\right) - \exp\left(-\eta N (1-q)^2\right) \right].
\]

For $q > \frac{1}{2}$, we have $(1-q)^2 < q^2$ and hence $\exp\left(-\eta N q^2\right) < \exp\left(-\eta N (1-q)^2\right)$, so $F'(q) > 0$.
For $q < \frac{1}{2}$, we have $(1-q)^2 > q^2$ and hence $\exp\left(-\eta N q^2\right) > \exp\left(-\eta N (1-q)^2\right)$, so $F'(q) < 0$.

Thus, for $q \neq \frac{1}{2}$, the truthful report $r = q$ is not a local maximum of $F(r)$. There exists some $r^* \neq q$ such that $F(r^*) > F(q)$. Consequently,
\[
\mathbb{E}[w_i^2 \mid S_D] > \mathbb{E}[w_i^2 \mid S_T].
\]

For $T=2$, the expected utility is
\[
\mathbb{E}[u_i \mid S] = 1 + \mathbb{E}[w_i^2 \mid S],
\]
so a higher expected $w_i^2$ directly implies a higher expected long-term utility. Therefore,
\[
\mathbb{E}[u_i \mid S_D] > \mathbb{E}[u_i \mid S_T],
\]
which proves that the EXP3 scheme is not truthful.

\section{Proof of Proposition~\ref{L2-}}\label{A6}

Since $g_i(\cdot)$ is strictly increasing, maximizing the utility
$u_i = g_i\!\left(\mathbb{E}\!\left[\sum_{t=1}^{T} w_i^t\right]\right)$
is equivalent to maximizing the expected cumulative weight
$\mathbb{E}\!\left[\sum_{t=1}^{T} w_i^t\right]$.
 According to our system model in Section~\ref{S5}, each worker believes that $p_j^t \sim \texttt{Bernoulli}(\mathcal{P}_i(y_{l_j}^t \succ y_{l_j'}^t|x_j^t))$ and $Pr(I_t = i) = \theta_i^t$, we have expectation on $w_{i}^{t+1}$ in \eqref{e4-} over $p_j^t$ and $I_t$ is
    \begin{align*}      
        &\mathbb{E}[w_{i}^{t+1}] \\
        =&  (1-\beta){\gamma_{i}^{t}}\frac{1}{m_t}\sum_{j=1}^{m_t}\bigg[ 1 -\alpha \bigg(1-\frac{\alpha}{\theta_i^t}\bigg) \bigg(
        \mathcal{P}_i(y_{l_j}^t \succ y_{l_j'}^t|x_j^t)\\
        &\ \ \ \ \ \ \  \ \ \ (\mathcal{\hat{P}}_i(y_{l_j}^t \succ y_{l_j'}^t|x_j^t)- 1)^2   
        + (1-\mathcal{P}_i(y_{l_j}^t \succ y_{l_j'}^t|x_j^t)) \\
        &\ \ \ \ \ \ \  \ \ \ (\mathcal{\hat{P}}_i(y_{l_j}^t \succ y_{l_j'}^t|x_j^t)- 0)^2\bigg) \bigg] + \beta \\
        =& (1-\beta){\gamma_{i}^{t}}\frac{1}{m_t}\sum_{j=1}^{m_t}\bigg[1 - \alpha\bigg(1-\frac{\alpha}{\theta_i^t}\bigg)\bigg( \big(\mathcal{\hat{P}}_i(y_{l_j}^t \succ y_{l_j'}^t|x_j^t) \\
        &\ \ \ \ \ \ \  \ \ \ - \mathcal{P}_i(y_{l_j}^t \succ y_{l_j'}^t|x_j^t)\big)^2 
        + \bigg(\mathcal{P}_i(y_{l_j}^t \succ y_{l_j'}^t|x_j^t) \\
        &\ \ \ \ \ \ \  \ \ \ - \mathcal{P}_i^2(y_{l_j}^t \succ y_{l_j'}^t|x_j^t)\bigg) \bigg)\bigg]+ \beta,
    \end{align*}
    which is maximized at $\mathcal{\hat{P}}_i^*(y_{l_j}^t \succ y_{l_j'}^t|x_j^t) = \mathcal{P}_i(y_{l_j}^t \succ y_{l_j'}^t|x_j^t)$ if $\alpha < \theta_i^t$. According to the choices of $\alpha$ and $\beta$ in Theorem~\ref{Thm1-}, we have $\theta_i^t \geq \frac{\beta}{N} = 2 \alpha > \alpha$, implying $\mathcal{\hat{P}}_i^*(y_{l_j}^t \succ y_{l_j'}^t|x_j^t) = \mathcal{P}_i(y_{l_j}^t \succ y_{l_j'}^t|x_j^t)$.  To obtain the largest possible cumulative weight, each worker will truthfully feedback his preference in the first time slot and all the following time slots because any deviation will lead to smaller weights of the next and all the following time slots. We then finish the proof.

\section{Proof of Theorem~\ref{Thm1-}}\label{A2}
According to Proposition~\ref{L2-}, we have $\mathcal{\hat{P}}_i(y_{l_j}^t \succ y_{l_j'}^t|x_j^t) = \mathcal{P}_i(y_{l_j}^t \succ y_{l_j'}^t|x_j^t)$ for all $j \in [m_t]$, $i \in [N]$ and $t \in [T]$. To derive a lower-bound on $\ln\frac{\sum_{i=1}^N \gamma_{i}^{T+1}}{\sum_{i=1}^N \gamma_{i}^1}$, we assume $\alpha N \leq \frac{\beta}{2}$ and have
\begin{align}
  \ln\frac{\sum_{i=1}^N \gamma_{i}^{T+1}}{\sum_{i=1}^N \gamma_{i}^1} =& \ln\bigg( \sum_{i=1}^N \gamma_{i}^{T+1} \bigg) - \ln\bigg( \sum_{i=1}^N \gamma_{i}^1 \bigg) \nonumber \\
  =& \ln \bigg( \sum_{i=1}^N \prod_{t=1}^T ( 1-\alpha \tilde{\ell}_i^t ) \bigg) - \ln N \nonumber \\
  \geq& \ln \bigg( \prod_{t=1}^T ( 1-\alpha\tilde{\ell}_{i^*}^t )  \bigg) - \ln N \nonumber \\
  =& \sum_{t=1}^T \ln \bigg( 1-\alpha \tilde{\ell}_{i^*}^t   \bigg) - \ln N \nonumber \\
  \geq& -\alpha\sum_{t=1}^T\tilde{\ell}_{i^*}^t -\alpha^2\sum_{t=1}^T\bigg(\tilde{\ell}_{i^*}^t\bigg)^2 - \ln N, \label{low''}
\end{align}
where we choose 
\begin{align*}
    \tilde{\ell}_{i}^t = \begin{cases}
        \frac{\ell_{i}^t(1-\alpha/\theta_i^t)}{\theta_i^t}, &\text{if} \ i = I_t, \\
        0, &\text{otherwise},
    \end{cases}
\end{align*}
$\alpha\tilde{\ell}_{i}^t \leq \alpha\frac{\ell_{i}^t(1-\alpha/\theta_{i}^t)}{\theta_{i}^t} \leq \alpha\frac{1}{\theta_{i}^t} \leq \alpha\frac{N}{\beta} \leq \frac{1}{2}$ and denote $i^*$ as the best worker in hindsight. Note that $\theta_i^t \geq \frac{\beta}{N}$ is equal to $(\gamma_i^t + \beta) N \geq \beta \sum_{i=1}^N \gamma_i^t$, which holds due to $\gamma_i^t + \beta > \beta > 0$ and $N > \sum_{i=1}^N \gamma_i^t$. The first inequality holds due to $\alpha\tilde{\ell}_{i}^t \leq \frac{1}{2}$ for all $i \in [N]$ and $t \in [T]$. The second inequality holds due to $\ln (1 - x) \geq -x -x^2$ for $x \leq \frac{1}{2}$.

To derive an upper-bound on $\ln\frac{\sum_{i=1}^N \gamma_{i}^{t+1}}{\sum_{i=1}^N \gamma_{i}^t}$, we have
\begin{align}
  \ln\frac{\sum_{i=1}^N \gamma_{i}^{t+1}}{\sum_{i=1}^N \gamma_{i}^t} 
  =& \ln \bigg( \frac{\sum_{i=1}^N \gamma_{i}^t \cdot (1-\alpha\tilde{\ell}_{i}^t)}{\sum_{i'=1}^N \gamma_{i'}^t} \bigg) \nonumber \\
  =& \ln \bigg( 1-\alpha\frac{\sum_{i=1}^N \gamma_{i}^t \cdot \tilde{\ell}_{i}^t}{\sum_{i'=1}^N \gamma_{i'}^t} \bigg) \nonumber \\
  \leq& -\alpha \frac{\sum_{i=1}^N \gamma_{i}^t \cdot \tilde{\ell}_{i}^t}{\sum_{i'=1}^N \gamma_{i'}^t} + \frac{1}{2}\alpha^2 \bigg(\frac{\sum_{i=1}^N \gamma_{i}^t \cdot \tilde{\ell}_{i}^t}{\sum_{i'=1}^N \gamma_{i'}^t}\bigg)^2 \nonumber \\
  \leq& -\alpha \frac{\sum_{i=1}^N \gamma_{i}^t \cdot \tilde{\ell}_{i}^t}{\sum_{i'=1}^N \gamma_{i'}^t} + \frac{1}{2}\alpha^2 \frac{\sum_{i=1}^N \gamma_{i}^t \cdot (\tilde{\ell}_{i}^t)^2}{\sum_{i'=1}^N \gamma_{i'}^t},  \label{upp''''}
\end{align}
where the first inequality holds due to $\ln (1 - \alpha x) \leq -\alpha x + \frac{1}{2}\alpha^2x^2$ for $x = \frac{\sum_{i=1}^N \gamma_{i}^t \cdot \tilde{\ell}_{i}^t}{\sum_{i'=1}^N \gamma_{i'}^t}$ and $\alpha x \leq 1/2$. The second inequality holds due to Jensen's inequality.
According to \eqref{upp''''}, we have
\begin{align}
    &\ln\frac{\sum_{i=1}^N \gamma_{i}^{T+1}}{\sum_{i=1}^N \gamma_{i}^1} \nonumber \\
    =& \ln\bigg(\frac{\sum_{i=1}^N \gamma_{i}^{T+1}}{\sum_{i=1}^N \gamma_{i}^t} \frac{\sum_{i=1}^N \gamma_{i}^t}{\sum_{i=1}^N \gamma_{i}^{t-1}} \cdot\cdots\cdot\frac{\sum_{i=1}^N \gamma_{i}^2}{\sum_{i=1}^N \gamma_{i}^1}\bigg) \nonumber \\
    =& \sum_{t=1}^T \ln\frac{\sum_{i=1}^N \gamma_{i}^{T+1}}{\sum_{i=1}^N \gamma_{i}^t} \nonumber \\
    \leq& -\alpha \sum_{t=1}^T \frac{\sum_{i=1}^N \gamma_{i}^t \cdot \tilde{\ell}_{i}^t}{\sum_{i'=1}^N \gamma_{i'}^t} + \frac{1}{2}\alpha^2 \sum_{t=1}^T\frac{\sum_{i=1}^N \gamma_{i}^t \cdot (\tilde{\ell}_{i}^t)^2}{\sum_{i'=1}^N \gamma_{i'}^t}. \label{upp''}
\end{align}
According to \eqref{low''} and \eqref{upp''}, we have
\begin{align*}
&-\alpha\sum_{t=1}^T\tilde{\ell}_{i^*}^t -\alpha^2\sum_{t=1}^T\bigg(\tilde{\ell}_{i^*}^t\bigg)^2 - \ln N  \\
\leq& -\alpha \sum_{t=1}^T \frac{\sum_{i=1}^N \gamma_{i}^t \cdot \tilde{\ell}_{i}^t}{\sum_{i'=1}^N \gamma_{i'}^t} + \frac{1}{2}\alpha^2 \sum_{t=1}^T\frac{\sum_{i=1}^N \gamma_{i}^t \cdot (\tilde{\ell}_{i}^t)^2}{\sum_{i'=1}^N \gamma_{i'}^t}.
\end{align*}
After re-arranging the above inequalities and dividing $\alpha$ on both sides, we have
\begin{align*}
   &\sum_{t=1}^T \frac{\sum_{i=1}^N \gamma_{i}^t \cdot \tilde{\ell}_{i}^t}{\sum_{i'=1}^N \gamma_{i'}^t}  - \sum_{t=1}^T\tilde{\ell}_{i^*}^t \\
   \leq& \frac{\ln N}{\alpha} + \frac{1}{2}\alpha \sum_{t=1}^T\frac{\sum_{i=1}^N \gamma_{i}^t \cdot (\tilde{\ell}_{i}^t)^2}{\sum_{i'=1}^N \gamma_{i'}^t} + \alpha\sum_{t=1}^T\bigg(\tilde{\ell}_{i^*}^t\bigg)^2.
\end{align*}
After taking expectation of $\tilde{\ell}_{i}^t$ on the above inequality, we have
\begin{align}\label{upper}
   &\mathbb{E}\bigg[\sum_{t=1}^T \frac{\sum_{i=1}^N \gamma_{i}^t \cdot \tilde{\ell}_{i}^t}{\sum_{i'=1}^N \gamma_{i'}^t}  - \sum_{t=1}^T\tilde{\ell}_{i^*}^t\bigg] \nonumber \\
   \leq& \frac{\ln N}{\alpha} + \frac{1}{2}\alpha \sum_{t=1}^T\sum_{i=1}^N \frac{\gamma_{i}^t }{\sum_{i'=1}^N \gamma_{i'}^t} \frac{(\ell_{i}^t)^2}{\theta_i^t} + \alpha\sum_{t=1}^T\frac{(\ell_{i^*}^t)^2}{\theta_i^t} \nonumber \\
   \leq& \frac{\ln N}{\alpha} + \frac{1}{2}\alpha NT  + \alpha\sum_{t=1}^T\frac{(\ell_{i^*}^t)^2}{\theta_i^t},
\end{align}
where the second inequality holds due to $\ell_i^t \in [0, 1]$ for $i \in [N]$, $\theta_i^t \geq \min\{\frac{\gamma_{i}^t }{\sum_{i'=1}^N \gamma_{i'}^t}, \frac{1}{N}\}$ and $\sum_{t=1}^T\sum_{i=1}^N \frac{\gamma_{i}^t }{\sum_{i'=1}^N \gamma_{i'}^t} \frac{1}{\theta_i^t} \leq \sum_{t=1}^T\sum_{i=1}^N \frac{\gamma_{i}^t }{\sum_{i'=1}^N \gamma_{i'}^t} \frac{1}{\min\{\frac{\gamma_{i}^t }{\sum_{i'=1}^N \gamma_{i'}^t}, \frac{1}{N}\}} = NT$.
We then derive a lower bound of the expectation of $\sum_{t=1}^T \frac{\sum_{i=1}^N \gamma_{i}^t \cdot \tilde{\ell}_{i}^t}{\sum_{i'=1}^N \gamma_{i'}^t}  - \sum_{t=1}^T\tilde{\ell}_{i^*}^t$. After taking expectation of $\hat{\ell}_{i}^t$, we have
\begin{align}\label{lower}
&\mathbb{E}\bigg[\sum_{t=1}^T \frac{\sum_{i=1}^N \gamma_{i}^t \cdot \tilde{\ell}_{i}^t}{\sum_{i'=1}^N \gamma_{i'}^t}  - \sum_{t=1}^T\tilde{\ell}_{i^*}^t\bigg] \nonumber \\
   =&\sum_{t=1}^T \sum_{i=1}^N\frac{ \gamma_{i}^t} {\sum_{i'=1}^N \gamma_{i'}^t} \ell_i^t\bigg(1-\frac{\alpha}{\theta_{i}^t}\bigg) - \sum_{t=1}^T\ell_{i^*}^t\bigg(1-\frac{\alpha}{\theta_{i}^t}\bigg) \nonumber \\
   =&\sum_{t=1}^T \sum_{i=1}^N\frac{ \gamma_{i}^t} {\sum_{i'=1}^N \gamma_{i'}^t} \ell_i^t - \sum_{t=1}^T\ell_{i^*}^t- \sum_{t=1}^T\sum_{i=1}^N\frac{ \gamma_{i}^t} {\sum_{i'=1}^N \gamma_{i'}^t}\frac{\alpha \ell_i^t}{\theta_i^t} \nonumber \\
   &+ \sum_{t=1}^T\ell_{i^*}^t \frac{\alpha}{\theta_i^t} \nonumber \\
   \geq&\sum_{t=1}^T \sum_{i=1}^N\frac{ \gamma_{i}^t} {\sum_{i'=1}^N \gamma_{i'}^t} \ell_i^t - \sum_{t=1}^T\ell_{i^*}^t - \alpha NT + \alpha\sum_{t=1}^T\frac{(\ell_{i^*}^t)^2}{\theta_i^t},
\end{align}
where the first equality holds due to $\ell_i^t \in [0, 1]$ for $i \in [N]$, $\theta_i^t \geq \min\{\frac{\gamma_{i}^t }{\sum_{i'=1}^N \gamma_{i'}^t}, \frac{1}{N}\}$ and $\sum_{t=1}^T\sum_{i=1}^N \frac{\gamma_{i}^t }{\sum_{i'=1}^N \gamma_{i'}^t} \frac{1}{\theta_i^t} \leq \sum_{t=1}^T\sum_{i=1}^N \frac{\gamma_{i}^t }{\sum_{i'=1}^N \gamma_{i'}^t} \frac{1}{\min\{\frac{\gamma_{i}^t }{\sum_{i'=1}^N \gamma_{i'}^t}, \frac{1}{N}\}} = NT$. According to \eqref{upper} and \eqref{lower}, we have
\begin{align*}
    \sum_{t=1}^T \sum_{i=1}^N\frac{ \gamma_{i}^t} {\sum_{i'=1}^N \gamma_{i'}^t} \ell_i^t - \sum_{t=1}^T\ell_{i^*}^t \leq \frac{3}{2}\alpha NT + \frac{\ln N}{\alpha}.
\end{align*}
Since $\theta_i^t = \frac{(1-\beta)\gamma_i^t + \beta}{(1-\beta)\sum_{i=1}^N \gamma_i^t + \beta N} < \frac{\gamma_i^t}{\sum_{i=1}^N \gamma_i^t} + \frac{1}{N}$, we further have
\begin{align*}
    &\sum_{t=1}^T \sum_{i=1}^N\bigg(\theta_i^t - \frac{1}{N}\bigg) \ell_i^t - \sum_{t=1}^T\ell_{i^*}^t \\
    <& \sum_{t=1}^T \sum_{i=1}^N\frac{ \gamma_{i}^t} {\sum_{i'=1}^N \gamma_{i'}^t} \ell_i^t - \sum_{t=1}^T\ell_{i^*}^t \\
    \leq& 3\alpha NT + \frac{\ln N}{\alpha},
\end{align*}
which is equal to
\begin{align*}
    \sum_{t=1}^T \sum_{i=1}^N\theta_i^t\ell_i^t - \sum_{t=1}^T\ell_{i^*}^t &\leq \sum_{i=1}^N \sum_{t=1}^T \frac{1}{N} \ell_i^t + 3\alpha NT + \frac{\ln N}{\alpha} 
\end{align*}
\begin{align*}
    &\leq \sum_{i=1}^N \sum_{t=1}^T \frac{2 \beta}{N} \ell_i^t + 3\alpha NT + \frac{\ln N}{\alpha} \\
    &\leq 2 \beta T + 3\alpha NT + \frac{\ln N}{\alpha},
\end{align*}
where the second inequality holds due to $\beta \leq \frac{1}{2}$ and the third due to $\ell_i^t \leq 1$. By taking $\beta =2\alpha N$, we have
\begin{align*}
 R_{\mathcal{M}}(T) &= \sum_{t=1}^T \sum_{i=1}^N\theta_i^t\ell_i^t - \sum_{t=1}^T\ell_{i^*}^t \leq 7\alpha NT + \frac{\ln N}{\alpha} \\
 &\leq 2\sqrt{7}\sqrt{NT\ln N} = \mathcal{O}(\sqrt{T})
\end{align*}
at $\alpha  = \sqrt{\frac{\ln N}{7NT}}$. Now let us check the condition for $\beta \leq \frac{1}{2}$, which is equal to $T > \frac{4}{\sqrt{7}}N \ln N$ and holds for $T \to \infty$. Note that $\beta = 2 \alpha N$ satisfies the condition of $\alpha \frac{N}{\beta} \leq \frac{1}{2}$. We then finish the proof.

\section{Proof of Proposition~\ref{prop-responsive}}\label{A-responsive}

By Proposition~\ref{L2}, workers report truthfully, so $\hat{\mathcal{P}}_i = \mathcal{P}_i$ and $L_i^s := \frac{1}{m_s}\sum_{j=1}^{m_s}(\mathcal{P}_i(y_{l_j}^s \succ y_{l_j'}^s|x_j^s) - p_j^s)^2 \in [0, 1]$ is worker $i$'s realized per-slot squared feedback loss, with expectation $\ell_i^s = \mathbb{E}[L_i^s]$.

\textit{Step 1: Multiplicative form of the expected weights.} By \eqref{e4}, $w_i^{s+1} = w_i^s(1 - \alpha L_i^s)$. Since $w_i^s$ is determined by $\{p_j^\sigma : \sigma < s\}$ and $L_i^s$ is determined by $\{p_j^s\}$, and since the realized states are independent across slots, $w_i^s$ and $L_i^s$ are independent. Taking expectations conditional on the weight at slot $t_0$ yields $\mathbb{E}[w_i^{s+1}] = \mathbb{E}[w_i^s] \cdot (1 - \alpha \ell_i^s)$. Iterating from $s = t_0$ for $\tau$ steps gives, for both workers,
\begin{align*}
    \mathbb{E}[w_i^{t_0+\tau}] = w_i^{t_0} \prod_{s=t_0}^{t_0+\tau-1} (1 - \alpha \ell_i^s), \quad \mathbb{E}[w_k^{t_0+\tau}] = w_k^{t_0} \prod_{s=t_0}^{t_0+\tau-1} (1 - \alpha \ell_k^s).
\end{align*}

\textit{Step 2: Per-slot growth lower bound.} For any $s \ge t_0$,
\begin{align*}
    \ln \frac{1 - \alpha \ell_i^s}{1 - \alpha \ell_k^s} = \int_{\ell_i^s}^{\ell_k^s} \frac{\alpha}{1 - \alpha u}\, du \ge \alpha (\ell_k^s - \ell_i^s) \ge \alpha \Delta,
\end{align*}
where the first inequality holds since $\frac{\alpha}{1 - \alpha u} \ge \alpha$ for $u \in [0, 1]$ under $\alpha < 1/2$, and the second follows from $\Delta := \min_{s \ge t_0}(\ell_k^s - \ell_i^s)$.

\textit{Step 3: Catch-up time.} Catching up after $\tau_{\textnormal{new}}$ updates means $\mathbb{E}[w_i^{t_0+\tau_{\textnormal{new}}}] \ge \mathbb{E}[w_k^{t_0+\tau_{\textnormal{new}}}]$, equivalently,
\begin{align*}
    \sum_{s=t_0}^{t_0+\tau_{\textnormal{new}}-1} \ln\bigg(\frac{1 - \alpha \ell_i^s}{1 - \alpha \ell_k^s}\bigg) \ge \ln\bigg(\frac{w_k^{t_0}}{w_i^{t_0}}\bigg).
\end{align*}
The sum has $\tau_{\textnormal{new}}$ terms, each lower-bounded by $\alpha\Delta$ from Step~2. Hence the catch-up condition is satisfied whenever $\tau_{\textnormal{new}} \alpha \Delta \ge \ln(w_k^{t_0}/w_i^{t_0})$, giving the upper bound
\begin{align*}
    \tau_{\textnormal{new}} \le \bigg\lceil \frac{\ln(w_k^{t_0}/w_i^{t_0})}{\alpha \Delta}\bigg\rceil = \mathcal{O}\bigg(\frac{1}{\alpha \Delta}\bigg) = \mathcal{O}\bigg(\sqrt{\frac{T}{\ln N}} \cdot \Delta^{-1}\bigg)
\end{align*}
after substituting $\alpha = \frac{2}{3}\sqrt{2\ln N/T}$ from Theorem~\ref{Thm1}. The bound depends on the weight ratio $w_k^{t_0}/w_i^{t_0}$ at the arrival slot but not on $t_0$ itself, so it does not grow with the existing worker's tenure prior to $t_0$. We finish the proof. \hfill $\square$

\section{Proof of Proposition~\ref{prop-robust}}\label{A-robust}

We prove parts (a) and (b) in turn.

\textit{Proof of part (a).} Fix worker $i$ and prompt $j$; we suppress the slot and prompt indices in this part for brevity. Under worker $i$'s Bernoulli belief $p_j^t \sim \text{Bernoulli}(\mathcal{P}_i)$, by the law of total probability under the symmetric $\epsilon$-flip,
\begin{align*}
    \Pr(\tilde{p}_j^t = 1) &= (1-\epsilon) \cdot \mathcal{P}_i + \epsilon \cdot (1 - \mathcal{P}_i) \\
    &= (1 - 2\epsilon) \mathcal{P}_i + \epsilon.
\end{align*}
Hence $\tilde{p}_j^t \sim \text{Bernoulli}(q_\epsilon)$ with $q_\epsilon := (1-2\epsilon)\mathcal{P}_i + \epsilon$. Intuitively, the symmetric flip contracts the signal toward $1/2$ by a factor of $(1-2\epsilon)$, and recovers $\mathcal{P}_i$ as $\epsilon \to 0$.

By the mean-variance decomposition of $\mathbb{E}[(\hat{\mathcal{P}}_i - \tilde{p}_j^t)^2]$ under $\tilde{p}_j^t \sim \text{Bernoulli}(q_\epsilon)$, the expected weight multiplier at slot $t$ under noisy verification is
\begin{align*}
    \mathbb{E}[w_i^{t+1}] = w_i^t \cdot \big[1 - \alpha (\hat{\mathcal{P}}_i - q_\epsilon)^2 - \alpha q_\epsilon (1 - q_\epsilon)\big],
\end{align*}
which is a strictly concave quadratic in $\hat{\mathcal{P}}_i$ uniquely maximized at $\hat{\mathcal{P}}_i^* = q_\epsilon$. The best-response deviation from truthful reporting is therefore
\begin{align*}
    |\hat{\mathcal{P}}_i^* - \mathcal{P}_i| = |q_\epsilon - \mathcal{P}_i| = \epsilon \cdot |1 - 2 \mathcal{P}_i| \le \epsilon.
\end{align*}

To bound the cumulative strategic gain, note that the per-slot gain from shifting the report from $\mathcal{P}_i$ to $q_\epsilon$ is at most $\alpha w_i^t (q_\epsilon - \mathcal{P}_i)^2 \le \alpha w_i^t \epsilon^2$. Using $w_i^t \le 1$ from \eqref{e4} and summing over $T$ slots,
\begin{align*}
    \sum_{t=1}^T \alpha w_i^t (q_\epsilon - \mathcal{P}_i)^2 \le \alpha \epsilon^2 T = \mathcal{O}(\epsilon^2 \sqrt{T}),
\end{align*}
after substituting $\alpha = \frac{2}{3}\sqrt{2\ln N/T}$ from Theorem~\ref{Thm1}. This is dominated by the $\mathcal{O}(\sqrt{T})$ regret term in Theorem~\ref{Thm1} and vanishes as $\epsilon \to 0$.

\textit{Proof of part (b).} Fix worker $i$ and prompt $j$. For any reported $\hat{\mathcal{P}}_i \in [0, 1]$ and any labels $p_j^t, \tilde{p}_j^t \in \{0, 1\}$,
\begin{align*}
    \big|(\hat{\mathcal{P}}_i - \tilde{p}_j^t)^2 - (\hat{\mathcal{P}}_i - p_j^t)^2\big| \le \mathbf{1}\{\tilde{p}_j^t \neq p_j^t\},
\end{align*}
since when the labels differ, $|(\hat{\mathcal{P}}_i - \tilde{p}_j^t)^2 - (\hat{\mathcal{P}}_i - p_j^t)^2| = |2\hat{\mathcal{P}}_i - 1| \cdot |\tilde{p}_j^t - p_j^t| \le 1$, and when they agree, the left-hand side is zero. Taking expectation and using $\Pr(\tilde{p}_j^t \neq p_j^t) \le \epsilon$,
\begin{align*}
    \big|\mathbb{E}[(\hat{\mathcal{P}}_i - \tilde{p}_j^t)^2] - \mathbb{E}[(\hat{\mathcal{P}}_i - p_j^t)^2]\big| \le \epsilon.
\end{align*}

Averaging over $j = 1, \dots, m_t$ and summing over $t = 1, \dots, T$, the aggregation loss in the first term of \eqref{e3} under noisy verification differs from the clean-case aggregation loss by at most $\epsilon T$. The same $\epsilon T$ bound applies to the best-expert benchmark loss in the second term of \eqref{e3}. By the triangle inequality and the clean-case regret bound in Theorem~\ref{Thm1},
\begin{align*}
   \mathbb{E}[R_{\mathcal{M}}(T)] \le 3\sqrt{\frac{T \ln N}{2}} + 2\epsilon T = \mathcal{O}(\sqrt{T}) + 2\epsilon T.
\end{align*}
Dividing by $T$,
\begin{align*}
    \frac{\mathbb{E}[R_{\mathcal{M}}(T)]}{T} \le \mathcal{O}\bigg(\frac{1}{\sqrt{T}}\bigg) + 2\epsilon,
\end{align*}
recovering the clean-case bound as $\epsilon \to 0$. We finish the proof. 


\end{document}